\documentclass[letterpaper, 10 pt, journal, twoside]{IEEEtran}
\usepackage{times}

\usepackage{amsmath}
\usepackage{amsthm}
\usepackage{amssymb}
\usepackage[utf8]{inputenc}
\usepackage[T1]{fontenc}
\usepackage{graphicx}
\usepackage{enumerate}
\usepackage{xcolor}
\usepackage{comment}
\usepackage[numbers]{natbib}
\usepackage{multicol}
\usepackage{hyperref}
\usepackage{tabularx}
\usepackage{makecell}
\usepackage{colortbl}
\usepackage{tablefootnote}
\usepackage{multirow}
\usepackage{paralist}
\usepackage{placeins}

\usepackage{amsmath,amsfonts}
\usepackage{algorithmic}
\usepackage{algorithm}
\usepackage{array}
\usepackage[caption=false,font=normalsize,labelfont=sf,textfont=sf]{subfig}
\usepackage{textcomp}
\usepackage{stfloats}
\usepackage{url}
\usepackage{verbatim}
\usepackage{graphicx}

\newtheorem{theorem}{Theorem}
\newtheorem{proposition}[theorem]{Proposition}

\theoremstyle{definition}
\newtheorem{definition}{Definition}
\hypersetup{
    colorlinks=true,
    linkcolor=blue,
    filecolor=magenta,      
    urlcolor=cyan,
}

\begin{document}

\title{Safety Embedded Differential Dynamic Programming Using Discrete Barrier States}

\author{Hassan Almubarak$^{1, 5}$, ~\IEEEmembership{Student Member,~IEEE,}, Kyle Stachowicz$^{2}$, Nader Sadegh$^{3}$, ~\IEEEmembership{Member,~IEEE,} and Evangelos A. Theodorou$^{4}$
\thanks{Manuscript received: August, 8, 2021; Revised November, 26, 2021; Accepted December, 27, 2021.}
\thanks{This paper was recommended for publication by Editor H. Kurniawati upon evaluation of the Associate Editor and Reviewers' comments.
This work was supported in part by the National Science Foundation, CPS Grant 1932288.}
\thanks{$^{1}$\scriptsize School of Electrical and Computer Engineering \tt{halmubarak@gatech.edu}}
\thanks{$^{2}$\scriptsize School of Computer Science {\tt kwstach@gatech.edu}}
\thanks{$^{3}$\scriptsize George W. Woodruff School of Mechanical Engineering {\tt sadegh@gatech.edu}}
\thanks{$^{4}$\scriptsize Daniel Guggenheim School of Aerospace Engineering \\ {\tt  evangelos.theodorou@gatech.edu}}
\thanks{\scriptsize Georgia Institute of Technology, Atlanta, GA, USA}
\thanks{$^{5}$\scriptsize Control and Instrumentation Engineering Department, King Fahd University of Petroleum \& Minerals, Dhahran, Saudi Arabia}%

\thanks{Digital Object Identifier (DOI):  10.1109/LRA.2022.3143301.}

}



\markboth{IEEE Robotics and Automation Letters. Preprint Version. Accepted December, 2021}
{Almubarak \MakeLowercase{\textit{et al.}}: Safety Embedded Differential Dynamic Programming} 


\maketitle

\begin{abstract}
Certified safe control is a growing challenge in robotics, especially when performance and safety objectives must be concurrently achieved. In this work, we extend the barrier state (BaS) concept, recently proposed for safe stabilization of continuous time systems, to safety embedded trajectory optimization for discrete time systems using discrete barrier states (DBaS). The constructed DBaS is embedded into the discrete model of the safety-critical system integrating safety objectives into the system's dynamics and performance objectives. Thereby, the control policy is directly supplied by safety-critical information through the barrier state. This allows us to employ the DBaS with differential dynamic programming (DDP) to plan and execute safe optimal trajectories. The proposed algorithm is leveraged on various safety-critical control and planning problems including a differential wheeled robot safe navigation in randomized and complex environments and on a quadrotor to safely perform reaching and tracking tasks. The DBaS-based DDP (DBaS-DDP) is shown to consistently outperform penalty methods commonly used to approximate constrained DDP problems as well as CBF-based safety filters.
\end{abstract}

\begin{IEEEkeywords}
Optimization and Optimal Control, Robot Safety, Constrained Motion Planning.
\end{IEEEkeywords}
\vspace{-2mm}
\section{Introduction}
\IEEEPARstart{S}{afety} in robotics, in its various forms - including collision avoidance, safe collaboration, etc. - is crucial to expanding the applicability of autonomous robots. With increasing demand for autonomy in various industries, this task is increasingly daunting even for known environments. Therefore, there is a clear need for provably safe controls. Yet, the difficulty in unifying safety and performance objectives usually calls for the trade-off between the objectives. To confront such a trade-off, this letter develops a technique to enforce safety in optimization-based controllers for discrete time nonlinear systems that guarantees safety as long as a solution exists. The letter builds on a recently proposed safety integrating technique for stabilization of continuous time systems \cite{Almubarak2021SafetyEC}, which enforces safety through embedding barrier states (BaS) into the model of the dynamical system. We extend the idea to trajectory optimization for discrete time nonlinear systems by developing a novel extension we term discrete barrier states (DBaS).

Safety, which can be verified through set invariance \cite{blanchini1999set}, is an increasingly important property of dynamical systems. The development of barrier certificates \cite{prajna2003barrier,prajna2004safety} formed an early approach to verification. Later, inspired by control Lyapunov functions and barrier certificates, \citet{wieland2007constructive} introduced control barrier functions (CBFs) to propose a feedback method of enforcing safety in continuous time systems. In an attempt to develop safe stabilization, \citet{ames2014control} and \citet{romdlony2014uniting} proposed spiritually similar, albeit distinct, CLF-CBF unification techniques. \citet{ames2014control} pioneered the CLF-CBF quadratic program (QP) paradigm which was further developed in \cite{ames2016CBF-forSaferyCritControl}. The CLF-CBF QP and the developed CBF have attracted researchers attention to be adopted in various control frameworks and robotic applications \cite{agrawal2017discrete,choi2020reinforcement,taylor2020adaptive,wang2018safe}. For discrete time systems, \citet{agrawal2017discrete} extended the notion of continuous time CBFs and CLF-CBF QPs to problems in discrete time. Nonetheless, discrete CBFs, which use reciprocal barrier functions, tend to be more restrictive than their continuous counterparts as the optimization problem may not be quadratic and is non-convex, which limits its applicability \cite{agrawal2017discrete}. Therefore, they proposed discrete time exponential barrier function (DECBFs) that solves the problem, which was generalized by \citet{ahmadi2019safe} and called discrete zeroing CBFs (DZCBFs), in analogy to its continuous time counterpart ZCBFs in \cite{ames2016CBF-forSaferyCritControl}.

\begin{figure} [t]
    \centering
    \includegraphics[trim=300 0 300 0, clip, height=0.4\linewidth]{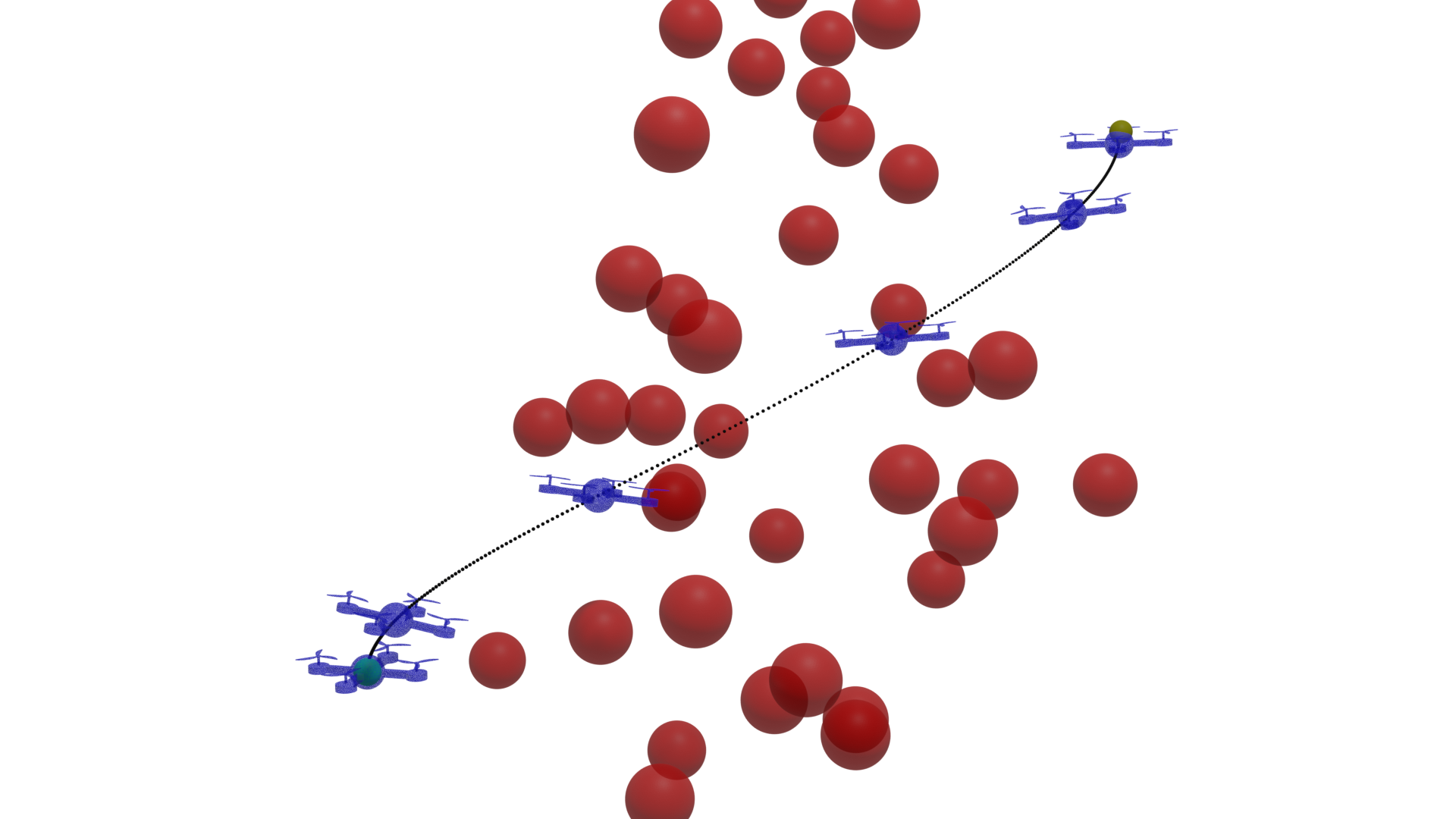}
    \includegraphics[trim=400 0 300 0, clip, height=0.4\linewidth]{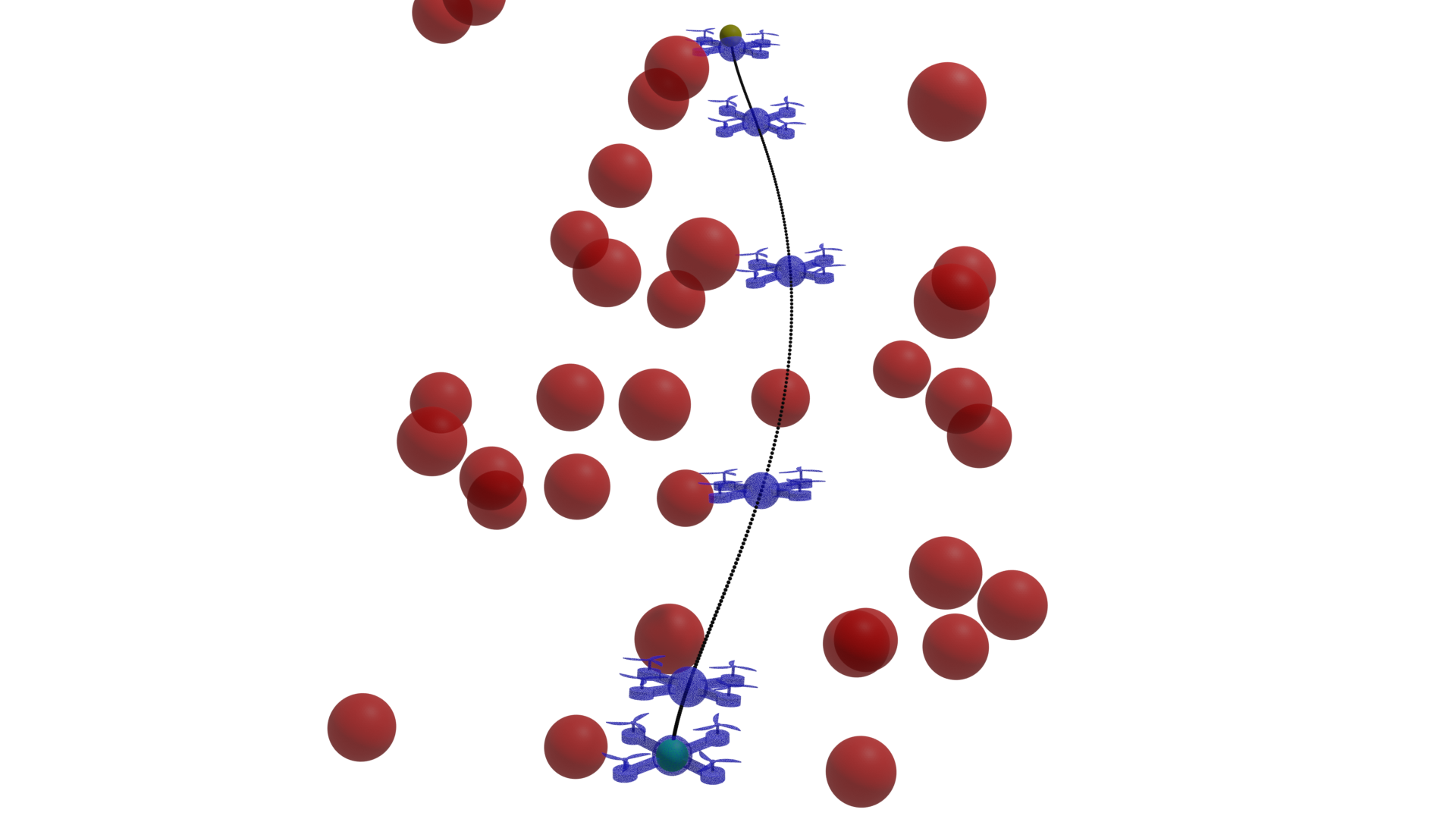}
    \caption{Two views of the quadrotor reaching problem with many spherical obstacles in the space. The proposed DBaS-DDP 
    safely performs the reaching task starting from the initial position (green) to the final position (yellow).}
\label{fig:quad_reaching}
\vspace{-5mm}
\end{figure}

To use CBFs in multi-objective controls, one must trade off between safety and performance objectives \cite{ames2016CBF-forSaferyCritControl}. Moreover, in the case of high relative-degree constraints, the problem becomes more challenging. Several methods have been developed to form higher-order CBFs, the most popular of these involving defining a new invariant safe set as the intersection of several invariant sets \cite{nguyen2016exponential,xiao2019control}. However, these approaches tend to unnecessarily restrict the allowable state space and are difficult to implement and tune. To avoid finding a high-order CBF altogether, some practical implementations \cite{marcus2020safe,long2021learning,xiao2020feasibility} instead apply an ad-hoc solution by modifying the CBF or the safety constraint to ensure a relative degree of $1$. However, this is likely to affect performance or safety of the resulting system with respect to the original objective and safety constraint.

Typically, when multiple safety constraints must be satisfied, multiple corresponding CBF inequality constraints are used \cite{marcus2020safe,long2021learning,xiao2020feasibility} and thus further relaxations could be needed. \citet{wang2016multi} proposed compositional barrier functions which can combine multiple CBFs into a single barrier function. However, this technique does not easily generalize, for example to learned or robust CBFs \cite{long2021learning}, and may create a CBF of ill-conditioned or high relative degree in some cases.

Constrained trajectory optimization is a challenging problem that has been revisited repeatedly in the literature. The differential dynamic programming (DDP) method can efficiently find optimal trajectories, but it is not straightforward to implement constraints. Contrarily, direct methods based on general nonlinear solvers can directly incorporate constraints, but at the expense of computational efficiency. \citet{murray1979constrainedDDP} describes an early approach to incorporating constraints into DDP, but only considers control constraints. This was later improved using active-set QP methods \cite{tassa2014control}, but the state-constrained case has remained a difficult open problem. One common approach is an application of the Augmented Lagrangian \cite{plancher2017constrained} \cite{howell2019altro} technique, which iteratively finds Lagrange multipliers for the constraint with first-order convergence. \citet{xie2017ddpNonlinearconstraints} proposed an active-set approach to the constrained DDP problem, which calculates active-set conditions in the backwards pass and solves a QP at each stage of the forwards pass. \citet{aoyama2020constrainedDDP} presented a related algorithm that switches online between an Augmented Lagrangian and an active-set method for faster global convergence. Finally, interior-point methods have been applied in \citet{pavlov2021interiorConstrainedDDP} to achieve local second-order convergence in the presence of nonlinear constraints. However, these algorithms often have difficulty with highly locally-nonlinear constraints and require substantial tuning and good-quality warm-starts to achieve satisfactory results.

\subsection{Contributions and Organization}
In this letter, we state the safety constraint formulation in \autoref{Section: Safety Constraint Formulation}. After that, we develop discrete barrier states (DBaS) to enforce safety for nonlinear discrete time systems in \autoref{Section: Discrete Time Barrier States}. Thereafter, a DBaS is embedded in the system's model forcing the control search to take place in the set of safe controls, which avoids compromising the performance or safety objectives. In addition, we show how to represent multiple constraints using a single DBaS. \autoref{Section: Safety Embedded DDP} states the constrained optimal control problem statement. Subsequently, we leverage the safety embedding technique with differential dynamic programming (DDP) to develop safe trajectory optimization. We show that the generated trajectories are guaranteed to be safe as long as the standard DDP convergence conditions are satisfied. 
We show the generality of our proposed framework in \autoref{Section: Safety Embedded DDP Examples} by applying DBaS-DDP to several systems including collision-avoidance problems for omnidirectional and differential wheeled robots, a cart-pole problem where motion is bounded by a fixed-length rail, and a variety of quadrotor tasks including safe trajectory tracking and reaching as in \autoref{fig:quad_reaching}. We compare DBaS-DDP with the penalty method DDP and with CBF when possible, and demonstrate that it exhibits improved performance and robustness characteristics in multiple extensive randomized experiments. Finally, concluding remarks and future directions are provided in \autoref{Section: Conclusion}.

\section{Safety Constraint Formulation} \label{Section: Safety Constraint Formulation}
Consider the discrete time nonlinear safety critical control system
\begin{equation} \label{discrete control system}
    x(k+1)=f(x(k),u(k))
\end{equation}
where $k \in \mathbb{Z}^+_0$ is the time step, $x(k) \in \mathcal{D} \subset \mathbb{R}^n$, $u(k)\in \mathcal{U} \subset \mathbb{R}^m$ and $f:\mathcal{D} \times \mathcal{U} \rightarrow \mathcal{D}$ is continuous. Throughout the work, we will use the subscript formulation to indicate the time step. 
For this system, consider the set $\mathcal{S}$ defined as the superlevel set of a smooth function $h: \mathcal{D} \rightarrow \mathbb{R}$ such that
\begin{equation} \begin{split} \label{safe set S}
    \mathcal{S}:= \{ x_k \in \mathcal{D} \ | \ h(x_k) \geq 0\} \\
    \mathcal{S}^{\circ}:= \{ x_k \in \mathcal{D} \ | \ h(x_k) > 0\} \\
    \partial \mathcal{S}:= \{ x_k \in \mathcal{D} \ | \ h(x_k) = 0\}
\end{split} \end{equation} 
where $ \mathcal{S}^{\circ}$ and $\partial \mathcal{S}$ are the interior and the boundary of the set $\mathcal{S}$, respectively. 
Let $\mathcal{S}^{\circ}$ be the safe set we desire the system's state to stay in. To enforce safety, one needs to satisfy the invariance property given by the following definitions. 
\begin{definition}  \label{invariant def}
The set $\mathcal{S}^{\circ} \subset \mathbb{R}^n$ is said to be forward invariant for the dynamical system $x(k+1)=f(x(k))$ if $\forall x(0) \in \mathcal{S}^{\circ}, x(k) \in \mathcal{S}^{\circ} \ \forall k\in \mathbb{Z}^+$. Equivalently,
\begin{equation} \label{safety condition}
    h(x_k) > 0 \ \forall k \geq 0 ; \ x(0) \in \mathcal{S}^{\circ} 
\end{equation}
We refer to this as the safety condition. 
\end{definition}
\begin{definition} \label{controlled invariant def}
The set $\mathcal{S}^{\circ} \subset \mathbb{R}^n$ is said to be \textit{controlled invariant} for the system in \eqref{discrete control system} if a continuous feedback controller $u_k=K(x_k)$ exists such that for the closed-loop system $x_{k+1}=f(x_k,K(x_k))$, the set $\mathcal{S}^{\circ}$ is forward invariant. Accordingly, the controller $u_k=K(x_k)$ is said to be safe. 
\end{definition}

To render $\mathcal{S}^{\circ}$ controlled invariant for the discrete control system \eqref{discrete control system}, we define the barrier function $B:\mathcal{S}^{\circ} \rightarrow \mathbb{R}$ 
\cite{ben2020lecturesoptimization}, to be a smooth function on the interior of $\mathcal{S}$ that goes to infinity as $x_k \in \mathcal{S}^{\circ}$ approaches a point of $\partial \mathcal{S}$. Mathematically,
$$
x_k \in \mathcal{S}^{\circ}, \tilde{x}\equiv \lim_{k\rightarrow \infty} x_k \in \partial\mathcal{S} \Rightarrow B(x_k) \rightarrow \infty, k\rightarrow \infty
$$
Examples of favored barrier functions with such properties over the set $\mathcal{S}$ defined by $h(x_k)$ include logarithmic barriers such as $B(h(x_k))=-\log(h(x_k))$ and $B(h(x_k))=-\log\Big(\frac{h(x_k)}{1+h(x_k)}\Big)$ and the Carroll barrier, also called the inverse barrier, $B(h(x_k))=\frac{1}{h(x_k)}$. 
Clearly, it is sufficient to force $B$ to be bounded to guarantee safety, i.e. keeping $x_k \in \mathcal{S}^{\circ} \ \forall k$. In light of this, \autoref{invariant def}, \autoref{controlled invariant def} and the properties of the barrier function $B$, the following proposition follows. 
\begin{proposition} \label{prop:safety}
A continuous feedback controller $u_k=K(x_k)$ is safe, that is it renders $\mathcal{S}^{\circ}$ controlled invariant, if and only if $B(h(x_0))<\infty \Rightarrow B(h(x_k))<\infty \ \forall k \in \mathbb{Z}^+$.
\end{proposition}
\begin{proof}
The proof follows directly from the definitions above. 

$\Rightarrow$. Suppose there exists a continuous control law $u_k=K(x_k)$ such that $\mathcal{S}^{\circ}$ is controlled invariant for w.r.t \eqref{discrete control system}. Then, by \autoref{controlled invariant def}, $\mathcal{S}^{\circ}$ is forward invariant w.r.t. the closed loop system $x_{k+1}=f(x_k,K(x_k))$ and consequently, by \autoref{invariant def} and the definition of $\mathcal{S}^{\circ}$, $h(x_k) >0 \ \forall k \geq 0$ implying 
$B(h(x_k))<\infty \ \forall k \geq 0$.

$\Leftarrow$. Assume $B(h(x_0))<\infty \Rightarrow B(h(x_k))<\infty \ \forall k \in \mathbb{Z}^+$ under the continuous control action $u_k=K(x_k)$. By the properties of the barrier functions, $h(x_k) >0 \ \forall k \geq 0$. Thus, by \autoref{invariant def}, $\mathcal{S}^{\circ}$ is forward invariant w.r.t. the closed loop system and hence $u_k$ is said to be safe by \autoref{controlled invariant def}.
\end{proof}
A main objective of this letter is to design a safety enforcing tool that allows us to 
avoid possible conflicts between control objectives and safety constraints without any relaxation. To achieve this goal, the safety constraint is embedded into the system's model used to achieve control performance objectives by means of discrete barrier states (DBaS). 

\section{Discrete Time Barrier States (DBaS)} \label{Section: Discrete Time Barrier States}
Let us define the barrier function to be $\beta(x_k):=B(h(x_k))$, that is for example for the inverse barrier, $\beta(x_k)=B(h(x_k))=\frac{1}{h(x_k)}$. Let $x^d$ be the desired state to be tracked. Define $w_k:= \beta(x_k) - \beta^{d}$, where $\beta^{d}=\beta(x^d)$. 
Without loss of generality, in the case of stabilization, $x^d$ will be a fixed point, e.g. the origin of the system. Consequently, we derive the discrete barrier state (DBaS) as
\begin{equation} \begin{split}
     w_{k+1} & =  B(h(f(x_k,u_k))  - \beta^{d} \\
\end{split} \end{equation}

In some robotic applications, e.g. in obstacle avoidance problems, it is more suitable to represent the safe region by a set of functions. In such a problem, increasing the dimension of the system by including too many barrier states may inflate the problem size and complexity. Multiple safety constraints can be represented with only one DBaS by combining the barrier functions. There are some drawbacks to combining barrier states in this way: the process may reduce the amount of information available to the controller as discussed in \autoref{Section: Safety Embedded DDP}, and we may lose access to some explicit information on the safety of the system with respect to certain constraints or obstacles which would be available with multiple barrier states. Therefore, representing multiple constraints with one barrier state introduces a trade-off between state dimension and information available to resulting feedback policies.

In the discrete setting, combining the barrier functions to create a barrier state for the discrete case is simpler than the continuous case \cite{Almubarak2021SafetyEC}. For $q$ constraints, the barrier function can be chosen to be $\beta(x)=\sum_{i=1}^{q} B(h^{(i)}(x_k))$, where $h^{(i)}(x_k)$ describes the $i^{\text{th}}$ region of interest. Consequently, a single DBaS can be constructed as 
\begin{align}\begin{split} \label{w(k+1) for multiple bfs}
   w_{k+1}= 
   \sum_{i=1}^{q} B\circ h^{(i)}\big(f(x_k,u_k)\big) - \beta^d
\end{split} \end{align}
where $\beta^d = \sum_{i=1}^{q} B\circ h^{(i)}(x^d)$. It must be noted that shifting the barrier state by $\beta^d$ is not necessary but ensures that the minimum lies at the desired set point which may be needed for some applications \cite{wills2004barrier}. Now, we append a vector of $p$ barrier states $w \in \mathcal{W} \subset \mathbb{R}^p$ to the model of the safety critical system \eqref{discrete control system} giving the safety embedded model 
\begin{equation} \label{safety embedded system}
     \hat{x}_{k+1} = \hat{f}(\hat{x}_k,u_k)
\end{equation}
where $\hat{x}_k= [x_k \ \ w_k]^{\rm{T}} \in \hat{\mathcal{D}} \subset \mathcal{D} \times \mathcal{W}$ and $\hat{f}:\hat{\mathcal{D}} \times \mathcal{U}\rightarrow \hat{\mathcal{D}}$ is a vector field representing the system's dynamics \eqref{discrete control system} and the barrier states' dynamics \eqref{w(k+1) for multiple bfs}. It must be noted that $\hat{f}$ is continuous due to continuity of $f$ and smoothness of $h$ and $B$. 
As a consequence of the development above, the forward-invariance of $\mathcal{S}^{\circ}$, i.e. safety of the safety-critical system, can be tied to the performance objectives of the safety embedded system \eqref{safety embedded system}. In other words, boundedness of the DBaS implies the generation of safe trajectories. Next, we use a finite horizon trajectory optimization technique, namely differential dynamic programming (DDP), to generate safely optimized trajectories.

\section{Safety Embedded DDP} \label{Section: Safety Embedded DDP}
In this section, we apply the proposed DBaS methodology to safe trajectory optimization by applying DDP \cite{mayne1966second,jacobson1967differential-PhDthesis,jacobson1970differential} to the safety embedded dynamics \eqref{safety embedded system}.
\subsection{Problem Statement}
We consider the finite horizon optimal control problem
\begin{equation}  \label{finite horizon cost functional}
    V_k(x)= \min_{U_k} \sum^{N-1}_{i=k} l(x_i,u_i) + \Phi(x_N)
\end{equation}
subject to the dynamical system \eqref{discrete control system} and the safety condition \eqref{safety condition}, where $U_k=\{u_k,u_{k+1}+ \dots + u_{N-1} \}$, $l$  and $\Phi$ are the the running cost and terminal cost respectively.
\subsection{Differential Dynamic Programming}
The well-known Bellman equation yields the following recurrence relation, with boundary condition $V_N = \Phi$:
\begin{equation} \label{Bellman eq}
V_k(x_k)=\min_{u_k} [l(x_k,u_k)+V_{k+1}(f(x_k,u_k))]
\end{equation}
The DDP algorithm iteratively solves the optimal control problem starting by expanding the value function around a nominal trajectory $(\bar{x},\bar{u})$ and solving \eqref{Bellman eq} to find a local feedback policy. Then, a new nominal trajectory for the system is computed forwards. The process is repeated until convergence.

Using the proposed DBaS technique to render $\mathcal{S}^{\circ}$ forward invariant, the constrained finite horizon optimal control problem reduces to an unconstrained optimal control problem that minimizes \eqref{finite horizon cost functional} subject to the safety embedded dynamics \eqref{safety embedded system}:
\begin{equation} \label{finite horizon cost functional for safe system}
    V_k(\hat{x})= \min_{U_k} \sum^{N-1}_{i=k} l(\hat{x}_i,u_i) + \Phi(\hat{x}_N)
\end{equation}
subject to $\hat{x}_{k+1} = \hat{f}(\hat{x}_k,u_k)$. Therefore, the associated Bellman equation can be given by $V_k(\hat{x}_k)=\min_{u_k} [l(\hat{x}_k,u_k)+V_{k+1}(\hat{f}(\hat{x}_k,u_k))]$. For the DDP equations used in the algorithm, define
\begin{align}
\begin{split} \label{eq:h_expansions}
&H_{\hat{x}_k}=l_{\hat{x}_k}  + {V^{\rm{T}}_{\hat{x}_{k+1}}} \hat{f}_{\hat{x}_k}, \ H_{u_k}= l_{u_k}  + {V^{\rm{T}}_{\hat{x}_{k+1}}} \hat{f}_{u_k} \\
&H_{\hat{x}\hat{x}_k}= l_{\hat{x}\hat{x}_k}  +   \hat{f}_{\hat{x}_k}^{\rm{T}} V_{\hat{x}\hat{x}_{k+1}} \hat{f}_{\hat{x}_k}  +  V_{\hat{x}_{k+1}} \hat{f}_{\hat{x}\hat{x}_k}  \\ 
&H_{uu_k}=  l_{uu_k}  +  \hat{f}_{u_k}^{\rm{T}} V_{\hat{x}\hat{x}_{k+1}} \hat{f}_{u_k}  + V_{\hat{x}_{k+1}} \hat{f}_{uu_k}  \\
&H_{\hat{x}u_k}=  l_{\hat{x}u_k}  +  \hat{f}_{\hat{x}_k}^{\rm{T}}  V_{\hat{x}\hat{x}_{k+1}} \hat{f}_{u_k}  + V_{\hat{x}_{k+1}} f_{\hat{x}u_k}
\end{split}
\end{align}
Then, the optimal variation may be given by
\begin{equation} \label{optimal var u}
\delta u^*_k = - H_{uu_k}^{-1} \big( H_{u_k}^{\rm{T}} +  H_{u \hat{x}_k} \delta \hat{x}_k\big) = {\rm{\mathbf{k}}}_k + {\rm{\mathbf{K}}}_k \delta \hat{x}
\end{equation}
where ${\rm{\mathbf{k}}}_k = - H_{uu_k}^{-1} H_{u_k}^{\rm{T}}, {\rm{\mathbf{K}}}_k =  - H_{uu_k}^{-1} H_{u \hat{x}_k}$ and $\delta x_k = x_k - \bar{x}_k, \delta u_k=u_k-\bar{u}_k$ represent deviations from the nominal state and control sequences, respectively. Now, as we need to minimize the expanded Bellman equation, setting each power of approximation to zero leads to the Riccati equations for  $V_k$, $V_{\hat{x}_k}$ and $V_{\hat{x}\hat{x}_k}$ that are solved to get
\begin{align} \begin{split}
\label{Backward Propagation equations of V}
& V_k = V_{k+1}-\frac{1}{2} H_{u_k} H_{uu_k}^{-1} H_{u_k}^{\rm{T}}\\
& V_{\hat{x}_k}= H_{\hat{x}_k} - H_{\hat{x}u_k} H_{uu_k}^{-1} H_{u_k} \\
& V_{\hat{x}\hat{x}_k}=\frac{1}{2} (H_{\hat{x}\hat{x}_k}-H_{\hat{x}u_k} H_{uu_k}^{-1} H_{u\hat{x}_k} )
\end{split} \end{align}
which are the equations used for the backward propagation. Consequently, one can compute $V_k$ and its gradient and Hessian along the states' trajectory as well as the optimal variation $\delta u$ backwards from $k=N-1$ to $k=1$ with the initialization $V_N(\hat{x}_N)=l_f(\hat{x}_N)$. Next, the new trajectory is computed forwards and the process is repeated until convergence.

Note that for the DDP problem to be well-defined, we must have $H_{uu_k} \succ 0$ \cite{jacobson1970differential}. 
For this to be the case it is sufficient to have that $V_{\hat{x}\hat{x}_k} \succ 0$ for all $k$. However, in the general case there is a distinct possibility that $l_{xx}$ is indefinite: it is perfectly reasonable for the cost function to be locally non-convex and in fact this is necessarily the case in the obstacle-avoidance problem. In contrast, the DBaS cost remains a convex function of the states, meaning that for a cost function $l(\hat{x}, u) = l_x(x, u) + l_w(w)$ its hessian $l_{\hat{x}\hat{x}}$ remains positive.
\begin{theorem}
\label{thm:posdef}
Assume the state is of the form $\hat{x} = [x \ w]^{\rm{T}}$ where $x$ is the real state of the system and $w$ is the barrier state. Further, let $\ell(\hat{x}, u)$ be of the form $\ell(\hat{x}, u) = \ell^x (x, u) + \ell^w (w)$, where $\ell_{ww} \succ 0$ and $\begin{pmatrix} \ell_{xx} & \ell_{xu} \\ \ell_{ux} & \ell_{uu} \end{pmatrix} \succ 0$, and assume second-order dynamics terms $f_{xx}$ are ignored.
Then, $V_{\hat{x}\hat{x}_k} \succ 0 \ \forall k$.
\end{theorem}
\begin{proof}
By induction: for the base case, we have by assumption that $V_{\hat{x}\hat{x}_N}$ is positive-definite, as it is simply $\Phi_{\hat{x}\hat{x}}$.

Then, we want to show that if $V_{\hat{x} \hat{x}_{k+1}}$ is positive definite we also have that $V_{\hat{x}\hat{x}_{k}}$ is positive definite. Examine the second-order expansion of the Bellman equation:
\begin{align*}
    \delta \hat{x}_k^{\rm{T}} V_{\hat{x}\hat{x}_k} \delta \hat{x_k} &= \min_{u} \begin{pmatrix}\delta x_k \\ \delta u_k\end{pmatrix}^{\rm{T}} \begin{pmatrix}\ell_{xx_k} & \ell_{xu_k} \\ \ell_{ux_k} & \ell_{uu_k}\end{pmatrix}\begin{pmatrix}\delta x_k \\ \delta u_k\end{pmatrix} \\
    &\quad+\delta w_k^{\rm{T}} \ell_{ww_k}\delta w_k + \delta \hat{x}_{k+1}^{\rm{T}} \bar V_{\hat{x}\hat{x}_{k+1}} \delta \hat{x}_{k+1} \\
    &\ge \delta \hat{x}_{k+1}^{\rm{T}} V_{\hat{x}\hat{x}_{k+1}} \delta \hat{x}_{k+1} > 0 \\
    \delta u_k^{\rm{T}}H_{uu_k}\delta u_k &= \delta u_k^{\rm{T}}\left(l_{uu} + f_{u}^{\rm{T}} V_{\hat{x} \hat{x}_{k+1}} f_{u}\right)\delta u \\
    &\ge (f_u \delta u)^{\rm{T}} V_{\hat{x}\hat{x}_{k+1}} (f_u\delta u_k) \ge 0
\end{align*}
So $V_{\hat{x}\hat{x}_k}$ is positive definite, and furthermore $H_{uu}$ is also positive definite. By induction, this holds for all $k$.
\end{proof}

A natural question is whether it is advantageous to include the barrier state in the model of the system's dynamics instead of simply adding the barrier to the cost function in the optimization problem as in some constrained DDP approaches, known as penalty methods. In those methods, the modified optimization problem given by the cost function $l(x, u) = \beta^2(x) + l'(x, u)$, where $\beta$ is a barrier function, appears at the surface level to be equivalent to the proposed safety embedded DDP formulation, which we term DBaS-DDP. 
While any local optimum for the DBaS based optimal control problem is also an optimum for the penalty method, the two mechanisms differ substantially in their interaction with the optimizer. Firstly, in many robotic applications, this new cost function is highly locally non-convex. \autoref{thm:posdef} explains part of the practical improvement seen when using barrier states over simple penalty methods: it moves some nonlinear terms from the cost function to the dynamics, removing local non-convexity from the problem. In this sense, the DBaS-DDP has a regularizing effect on the algorithm when applied to highly non-convex cost functions.
It is worth mentioning that for the experiments presented in \autoref{Section: Safety Embedded DDP Examples}, the DBaS-DDP did not need any explicit regularization (i.e. adding a multiple of identity to $H_{uu_k}$ so it is positive definite), unlike the penalty method. Secondly, in the case of DBaS-DDP the optimizer has richer information and can better anticipate the progression of the cost in the forward pass:  by embedding the barrier state in the dynamics, the feedback mechanism present in the forwards pass has access to the exact value of the barrier state according to the nonlinear dynamics rather than a quadratic approximation. In this sense, the optimization process can be considered a joint optimization over the real state, barrier state, and controls as decision variables, yielding a smoother cost landscape.

Under certain conditions, with the incorporation of line search, standard DDP is able to guarantee that a single iteration of DDP will improve the trajectory's cost. Similarly, we show that in our formulation, a single iteration of DDP-DBaS is able to find a \textit{safe} trajectory with improved cost.

\begin{theorem}[Improvement of Safe Trajectory]
Let $(\bar x, \bar u)$ be a safe nominal trajectory, and let 
$ \delta u = \epsilon {\rm{\mathbf{k}}} + {\rm{\mathbf{K}}} \delta \hat{x}
$. If $\delta u_k$ is nonzero for some $k$, then there exists some $0 < \epsilon \le 1$ such that 
for the objective function $J=\sum^{N-1}_{k=1} l(\hat{x}_k,u_k) + \Phi(\hat{x}_N)$, $J(\bar x + \delta x, \bar u + \delta u) < J(\bar x, \bar u)$ and $\bar x + \delta x, \bar u + \delta u$ is a safe trajectory.
\end{theorem}
\begin{proof}
Because $\bar x, \bar u$ is safe and the safe set is open, there exists some neighborhood $U$ of safe trajectories near $\bar x, \bar u$.
Define $J(\pi)$ as the objective function: $J(\pi) = \sum_{k=1}^{N-1}l(\hat x_k, u_k) + \Phi(\hat x_N)$ with the update rule $\delta u_k = u_k - \bar u_k = 
 {\rm{\mathbf{K}}}_k \delta \hat{x}_k + \epsilon {\rm{\mathbf{k}}}_k$.
Find partial derivatives with respect to $u_k$ using Bellman's principle $\Delta J= [l(\hat x_k, u_k) + V(f(\hat x_k, u_k))] - [l(\bar x_k, \bar u_k) + V(f(\bar x_k, \bar u_k))]$ to get $
     \frac{\partial}{\partial u_k} \Delta J = H_{u_k}, \
    \frac{\partial^2}{\partial u_k^2} \Delta J = H_{uu_k}
    $.

Rewrite the objective function $\Delta J$ as a function of the parameter $\epsilon$, with Taylor expansion around zero:
\[\Delta J(\epsilon) = \sum_{k=1}^{N-1}\left[H_{u_k}\frac{\partial u_k}{\partial \epsilon}\epsilon + \frac{1}{2}\frac{\partial u_k}{\partial\epsilon}^{\rm{T}}H_{uu_k}\frac{\partial u_k}{\partial \epsilon}\epsilon^2\right] + \mathcal{O}(\epsilon^{3})\]
Finally, because $u_k = \bar u_k + 
{\rm{\mathbf{K}}}_k \delta \hat{x}_k + \epsilon {\rm{\mathbf{k}}}_k$, we have $\frac{\partial u_k}{\partial \epsilon} = k_k = -Q_{uu}^{-1}$:
\begin{align*}
    \Delta J(\epsilon) &= \sum_{k=1}^{N-1}\left[\frac{\epsilon^2}{2}H_{u_k}H_{uu_k}^{-1}H_{u_k}^{\rm{T}}-\epsilon 
    H_{u_k}H_{uu_k}^{-1}H_{u_k}^{\rm{T}}\right] + \mathcal{O}(\epsilon^{3}) \\
    &= (\frac{1}{2}\epsilon^2-\epsilon)\sum_{k=1}^{N-1}\left[H_{u_k}H_{uu_k}^{-1}H_{u_k}^{\rm{T}}\right] + \mathcal{O}(\epsilon^{3})
\end{align*}
By \autoref{thm:posdef}, we know that $H_{uu_k}$ is positive definite. Then, the summation is positive and so for $\epsilon \le 1$ the lower-order terms are negative. In addition, by making $\epsilon$ small we can reduce the higher-order terms to be arbitrarily small compared to the reduction in cost.

Finally, because there is an open neighborhood $U$ of safe trajectories around the nominal trajectory, by making $\epsilon$ small we will find a trajectory within the safe set. In particular, because a trajectory is safe if and only if it has a bounded cost by design, any $\epsilon$ that leads to a reduction in cost yields a safe trajectory.
\end{proof}

\section{Safety Embedded DDP Application Examples} \label{Section: Safety Embedded DDP Examples}
In this section, we conduct qualitative and quantitative comparisons between our proposed algorithm and some commonly-used methods of enforcing safety. Namely, we compare against the penalty-DDP method, in which a barrier cost is added to the state cost function (similar to the DBaS-DDP cost) as described earlier, and the traditional CBF-based method by wrapping DDP with a CBF filter.

Results from all experiments, including success rates and mean costs, are recorded in \autoref{tab:robustness}. In this table costs are normalized to allow for comparison between experiments, by expressing each cost as a fraction of the cost achieved using DBaS-DDP. In each case, DBaS-DDP achieves the lowest cost and the highest success rate.

Convergence speed is detailed in \autoref{tab:timing}. M.I. indicates the average number of iterations needed to solve a problem (i.e. until the trajectory reaches a threshold of the goal) and C.I. indicates the number of iterations before the algorithm has converged (the difference in cost between successive iterations drops below $10^{-3}$). Only successfully completed problems are included in this timing table, as including failures as the maximum iteration count would artificially inflate the number of iterations for the penalty-DDP method (which often fails to reach the goal). It can be seen that DBaS-DDP takes a few more iterations on average but brings substantially higher success rates and lower costs as shown in \autoref{tab:robustness}.

In all experiments in this section, we pick a quadratic cost function $J=\sum^{N-1}_{i=k} x_k^{\rm{T}} Q x_k + w_k^{\rm{T}} q_w w_k + u_k^{\rm{T}} R u_k + x_N^{\rm{T}} S x_N + w_N^{\rm{T}} s_w w_N$ with $Q=\mathbf{0}_{n\times n}, R=0.005I$. It is worth mentioning that tuning for DDP parameters specific to the safety embedded case is not required; the parameters from the unconstrained case works well with $q_w > 0$ to ensure safety. The continuous time systems were discretized using
Euler methods with $\Delta t=0.02$. All problems are initialized with a steady-state nominal trajectory (hovering for quadrotor and zero input for other examples). Note that in all experiments, the relative degree of the safety constraint is higher than 1 and is in some cases ill-conditioned, meaning that CBFs are either have limited behavior or impossible to apply. Because of this, we only consider CBFs in the planar double-integrator and cart-pole examples. To address the possibility of stepping across a barrier in a single discrete step, the interior of a barrier is assumed to have infinite cost, ensuring that such trajectories are rejected by line search. This effect can also be addressed with a finer discretization.

\begin{table}[]
\caption{Comparison of DBaS-DDP, the penalty method, and CBF-QP safety filtering.}
\centering
\begin{tabular}{|c|cc|cc|cc|}
\hline
\multirow{2}{*}{Experiment} & \multicolumn{2}{c|}{DBaS}                          & \multicolumn{2}{c|}{Penalty}    & \multicolumn{2}{c|}{CBF}        \\ \cline{2-7} 
                            & \multicolumn{1}{c|}{\%}           & Cost           & \multicolumn{1}{c|}{\%}   & Cost  & \multicolumn{1}{c|}{\%} & Cost  \\ \hline
Point Robot         & \multicolumn{1}{c|}{\textbf{95}}  & \textbf{1.00x} & \multicolumn{1}{c|}{77}   & 1.17x & \multicolumn{1}{c|}{79} & 2.54x \\
Diff. Wheeled       & \multicolumn{1}{c|}{\textbf{82}}  & \textbf{1.00x} & \multicolumn{1}{c|}{21.7} & 4.69x  &\multicolumn{1}{c|}{-}  & -     \\
Cart-pole           & \multicolumn{1}{c|}{-}            & \textbf{1.00x} & \multicolumn{1}{c|}{-}    & 1.26x & \multicolumn{1}{c|}{-}   &  1.74x     \\
Quad (1 obstacle)           & \multicolumn{1}{c|}{\textbf{100}} & \textbf{1.00x} & \multicolumn{1}{c|}{81}   & 1.57x & \multicolumn{1}{c|}{-}  & -     \\
Quad (3 obstacles)          & \multicolumn{1}{c|}{\textbf{90}}  & \textbf{1.00x} & \multicolumn{1}{c|}{37}   & 3.84x & \multicolumn{1}{c|}{-}  & -     \\
Quad (Random)       & \multicolumn{1}{c|}{\textbf{96}}  & \textbf{1.00x} & \multicolumn{1}{c|}{59}   & 1.87x & \multicolumn{1}{c|}{-}  & -     \\ \hline
\end{tabular}
\label{tab:robustness}
\end{table}

\begin{table}[h!]
\caption{Timing and convergence information. M.I. is for minimum iterations needed to reach the goal and C.I. is the iterations needed to achieve convergence.}
\centering
\begin{tabular}{|c|cc|cc|cc|}
\hline
\multirow{2}{*}{Experiment} & \multicolumn{2}{c|}{Unconstrained} & \multicolumn{2}{c|}{Penalty} & \multicolumn{2}{c|}{DBaS}\\ \cline{2-7} 
                            & \multicolumn{1}{c|}{M.I.} & C.I.   
                            & \multicolumn{1}{c|}{M.I.} & C.I.   
                            & \multicolumn{1}{c|}{M.I.} & C.I.   \\ \hline
Point Robot\tablefootnote{Unconstrained point robot is an LQ problem and completes in one step.}
                            & \multicolumn{1}{c|}{1}    & 1
                            & \multicolumn{1}{c|}{1.39} & 3.87 
                            & \multicolumn{1}{c|}{2.82} & 10.47 \\
Diff. Wheeled               & \multicolumn{1}{c|}{10.48}& 19.48       
                            & \multicolumn{1}{c|}{11.19}& 24.40       
                            & \multicolumn{1}{c|}{13.96}& 25.08 \\
Cart-pole\tablefootnote{Cart-pole is a qualitative example, i.e. no randomized experiments.}
                            & \multicolumn{1}{c|}{2}   &  17     
                            & \multicolumn{1}{c|}{284} &  408     
                            & \multicolumn{1}{c|}{37}   &  63    \\
Quadrotor                   & \multicolumn{1}{c|}{1}    &  12
                            & \multicolumn{1}{c|}{1.24} &  16.25
                            & \multicolumn{1}{c|}{1.93} &  31.82     \\ \hline
\end{tabular}
\label{tab:timing}
\end{table}

\subsection{Planar Double-Integrator (point robot)}
As a simple proof-of-concept we apply DBaS-DDP to the planar double-integrator problem. In this problem, an omnidirectional robot navigates from an initial position to a goal. The barrier state was defined, by \eqref{w(k+1) for multiple bfs}, such that the safe set is the exterior of several circles each with safe set defined by $h_i(x) = \lVert x - o_i \rVert^2 - r_i^2$. We pick $S = \text{diag}(4000, 4000, 400, 400)$ for terminal cost.

\autoref{fig:double_integrator} shows a collision-free course planned by DBaS-DDP starting from the initial point $(0,0)$ to the goal $(3,3)$. The figure also shows solutions generated by the penalty method and a higher order CBF \cite{nguyen2016exponential} QP that is wrapped around the unconstrained solution. Two choices of parameter ($q_w$ for DDP-based methods and $\alpha$ for CBF) are shown for each method. Small barrier penalization $q_w$ causes tighter constraint tolerance near the obstacle, but can cause the penalty method to easily get stuck in local minima. Large values of $\alpha$ in CBFs cause smaller nominal deviation from the trajectory but require sharp inputs when near a barrier.

To empirically test the results of \autoref{thm:posdef}, we compare the eigenvalues of $H_{uu_k}^{-1}$ between the naive penalty method and DBaS-DDP. We find that for the penalty method, we have a minimum eigenvalue (across an entire run of DDP) of $-0.173$, meaning that although the penalty method is able to find some non-intersecting path in this case, it is despite numerical ill-conditioning. Comparatively, using DBaS-DDP the minimum eigenvalue is $0.1$, meaning that $H_{uu_k}$ is numerically well-conditioned across the entire sweep.

We also perform robustness testing with randomized obstacle configurations. In these tests the robot is to safely move from $(0, 0)$ to $(3, 3)$ while avoiding circular obstacles randomly generated in the box with corners at $(3, -2), (5, 0), (0, 5), (-2, 3)$. Success is defined as a trajectory that is able to reach within 0.3 units of the goal. \autoref{tab:robustness} shows the results of penalty-DDP, DBaS-DDP, and CBF over a uniform distribution of obstacle count from 1 to 10 and \autoref{fig:DI_success_rate} shows these results broken down by number of obstacles.
\begin{figure}[h]
    \vspace{-3mm}
    \centering
    \hspace*{-4mm}
    \subfloat{\includegraphics[trim=50 0 40 20, clip, width=0.38\linewidth]{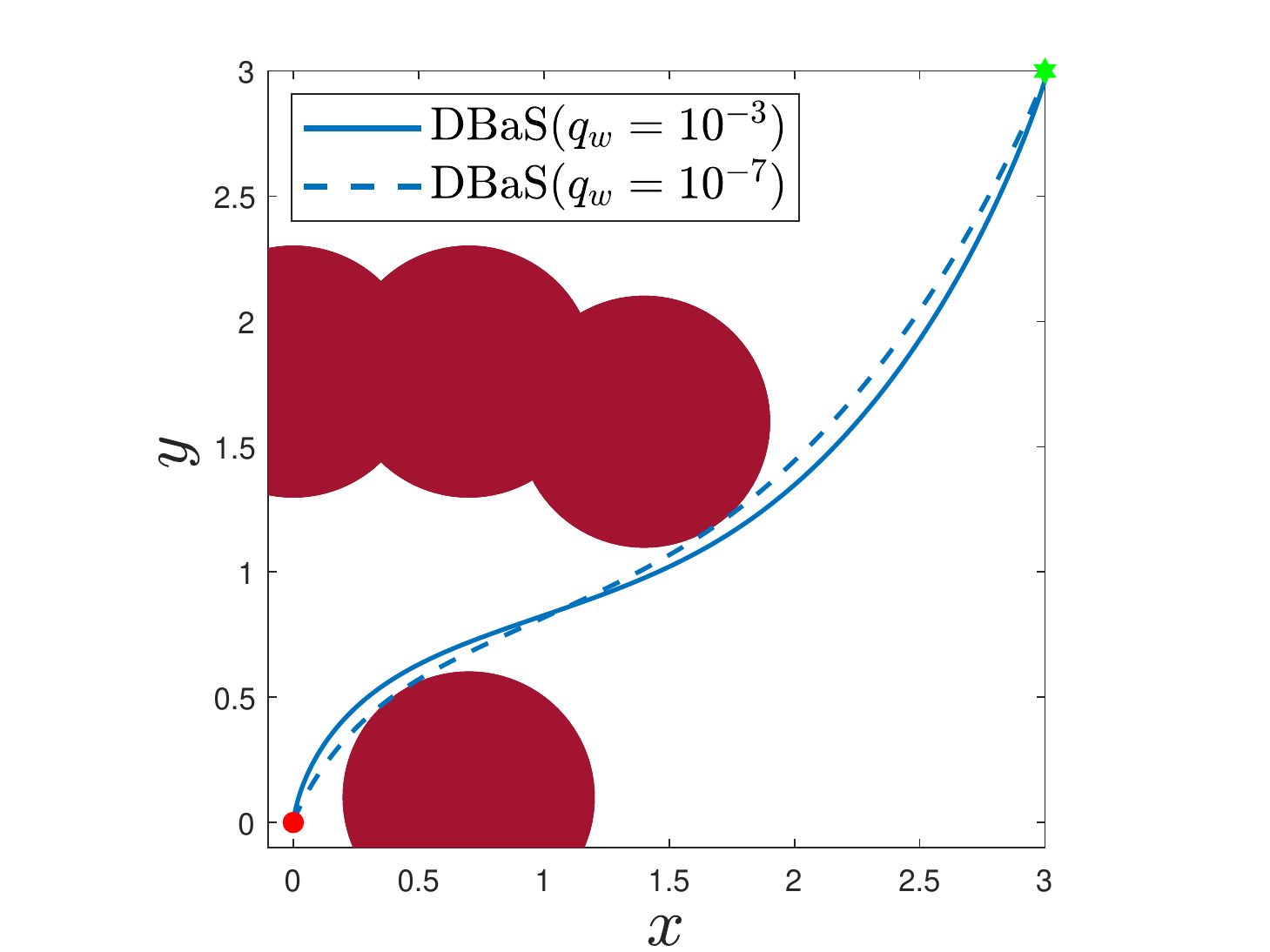}}
    \hspace*{-4mm}
    \subfloat{\includegraphics[trim=50 0 40 20, clip, width=0.38\linewidth]{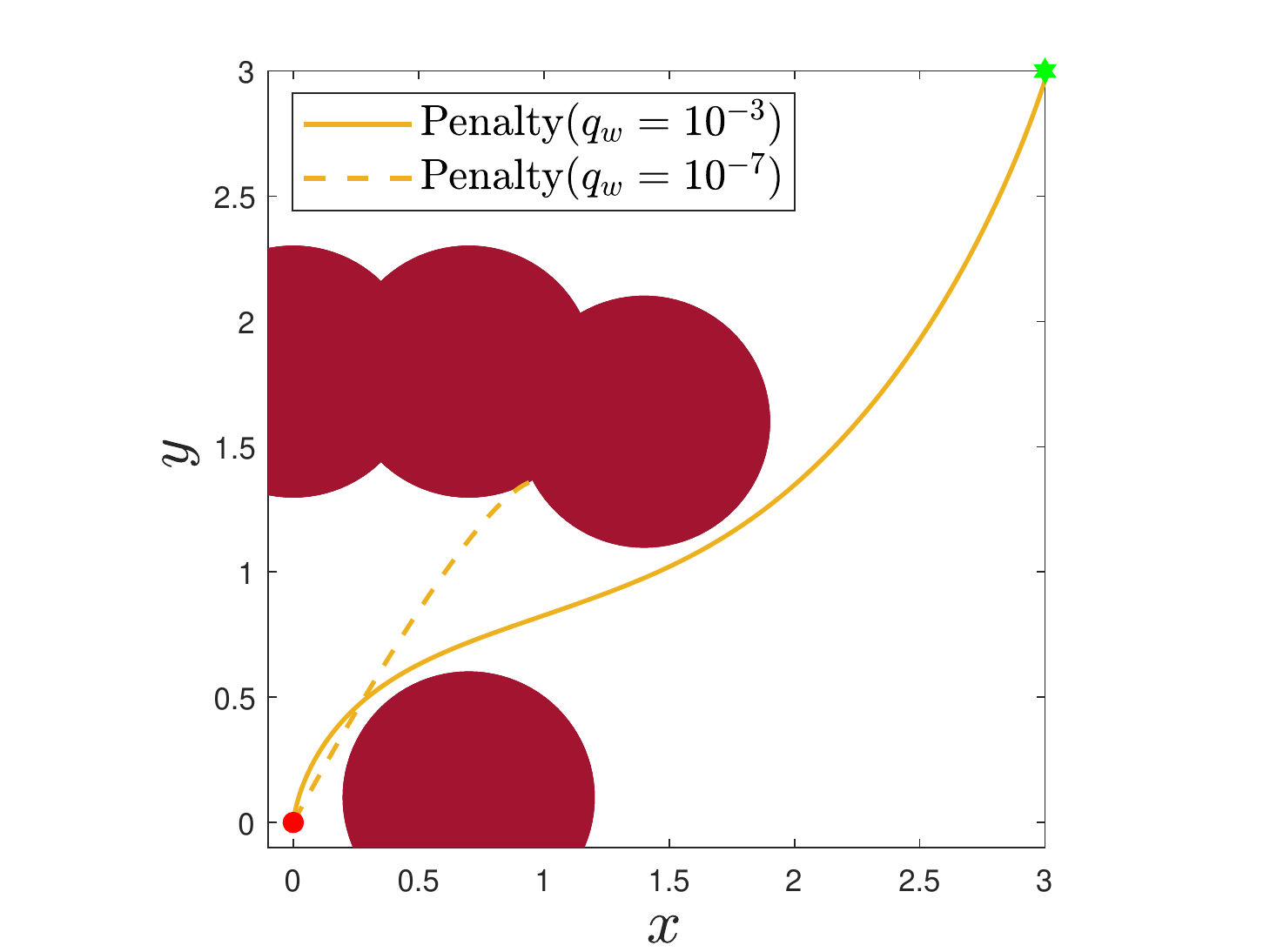}}
    \hspace*{-4mm}
    \subfloat{\includegraphics[trim=50 0 40 20, clip, width=0.38\linewidth]{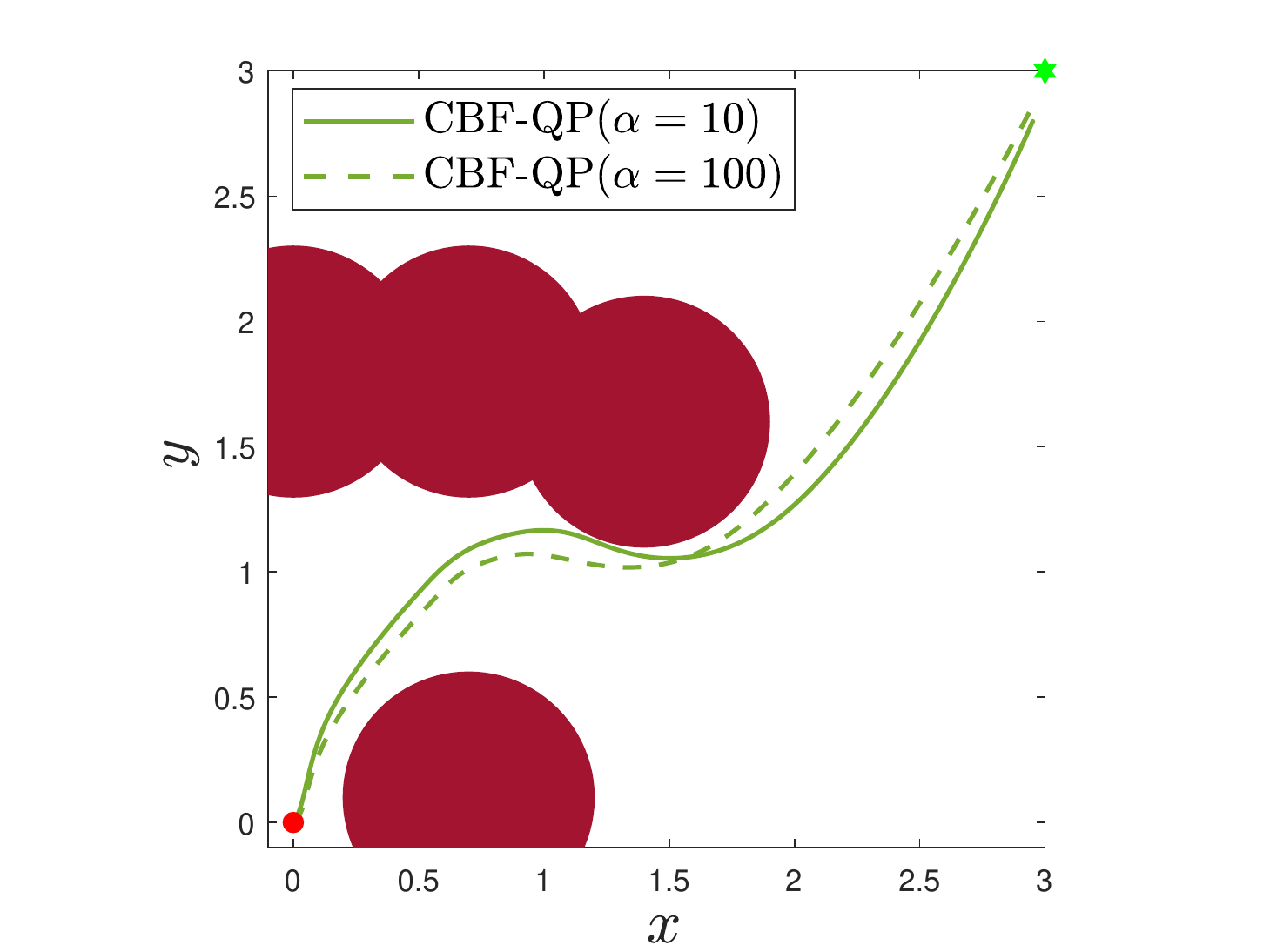}}
    \caption{Planar double-integrator navigation with DBaS (left), penalty (center) and a higher order CBF \cite{nguyen2016exponential} (right) with different parameters.}
    \label{fig:double_integrator}
    \vspace{-5mm}
\end{figure}

\begin{figure} [htb]
    \centering
    \includegraphics[trim=0 5 40 0, width=0.7\linewidth]{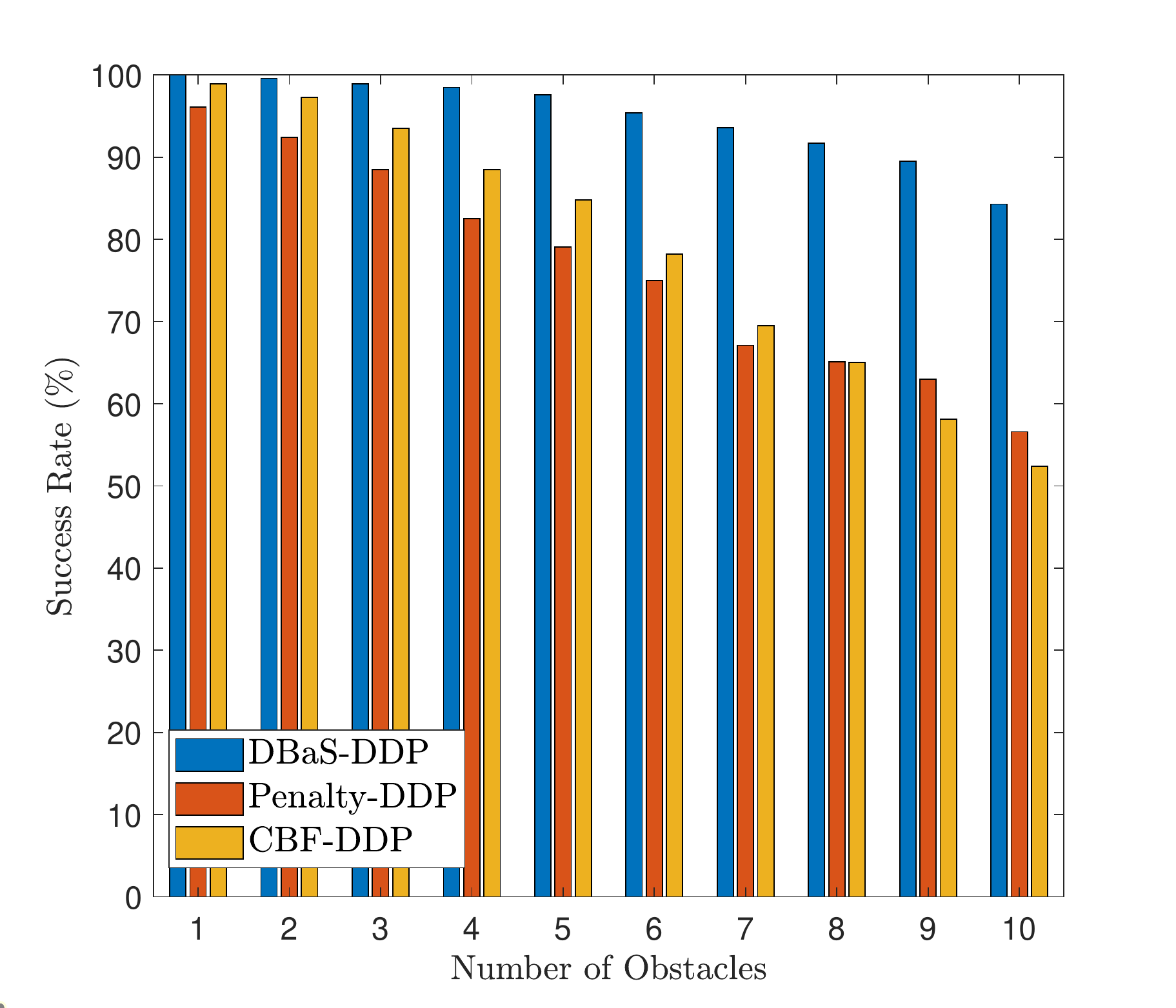}
    \caption{Success rates of DBaS-DDP (blue), penalty method (red), and CBF (yellow) on planar double integrator example with varying obstacle counts.}
\label{fig:DI_success_rate}
\end{figure}

\subsection{Cart-pole Swing up}
We demonstrate the applicability of DBaS-DDP to non-collision-avoidance problems using the cart-pole system, in which the controller must swing up the pole in 3 seconds by moving the cart while adhering to the safety constraints on the cart's position. We use the system dynamics from \cite{marcus2020safe}, but with a tighter constraint in the cart's position and with no modification of the safety constraint to obtain a low relative degree. Namely, 
we define our safe set by $h(x) = x^2_{\text{lim}} - x^2$ where $x_{\text{lim}}=1.5$. We pick $R=0.05, q_w=10^{-3}$ and $S = \text{diag}(50,800,10,10)$. For the high order CBF, the vector $\alpha =[100 \ \ 200]$ gave the best feasible results. Results are shown in \autoref{fig:Cart pole}. Comparatively, we conclude that:
\begin{compactitem}
    \item Unconstrained DDP violates the safety constraint.
    \item The CBF-based method maintains safety but needed a great deal of tuning to ensure feasibility and complete the task, and requires high control input.
    \item Penalty-DDP also yields a safe solution but required a lot more iterations to find a solution and converge.
\end{compactitem}
Only DBaS-DDP is able to find the optimal solution for the problem. This is reflected in its comparatively lower cost as shown in \autoref{tab:robustness}. It also converges in much fewer iterations than penalty-DDP as shown in \autoref{tab:timing}. 
\begin{figure}[h]
    \centering
    \vspace{-5mm}
    \hspace*{-7mm}
    \subfloat{\includegraphics[trim=15 0 30 0, clip, width=0.5\linewidth]{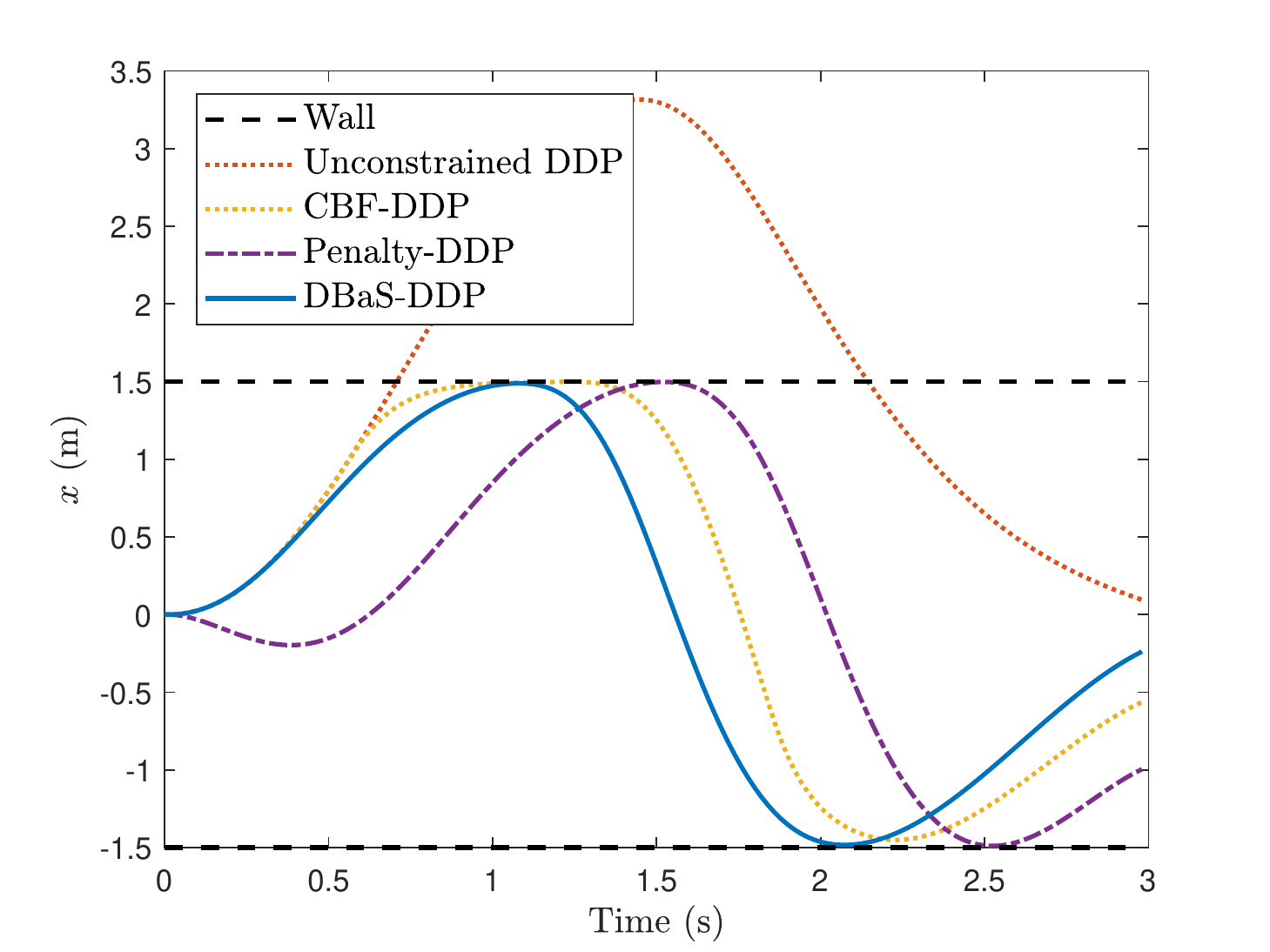}}
    \subfloat{\includegraphics[trim=15 0 30 0, clip, width=0.5\linewidth]{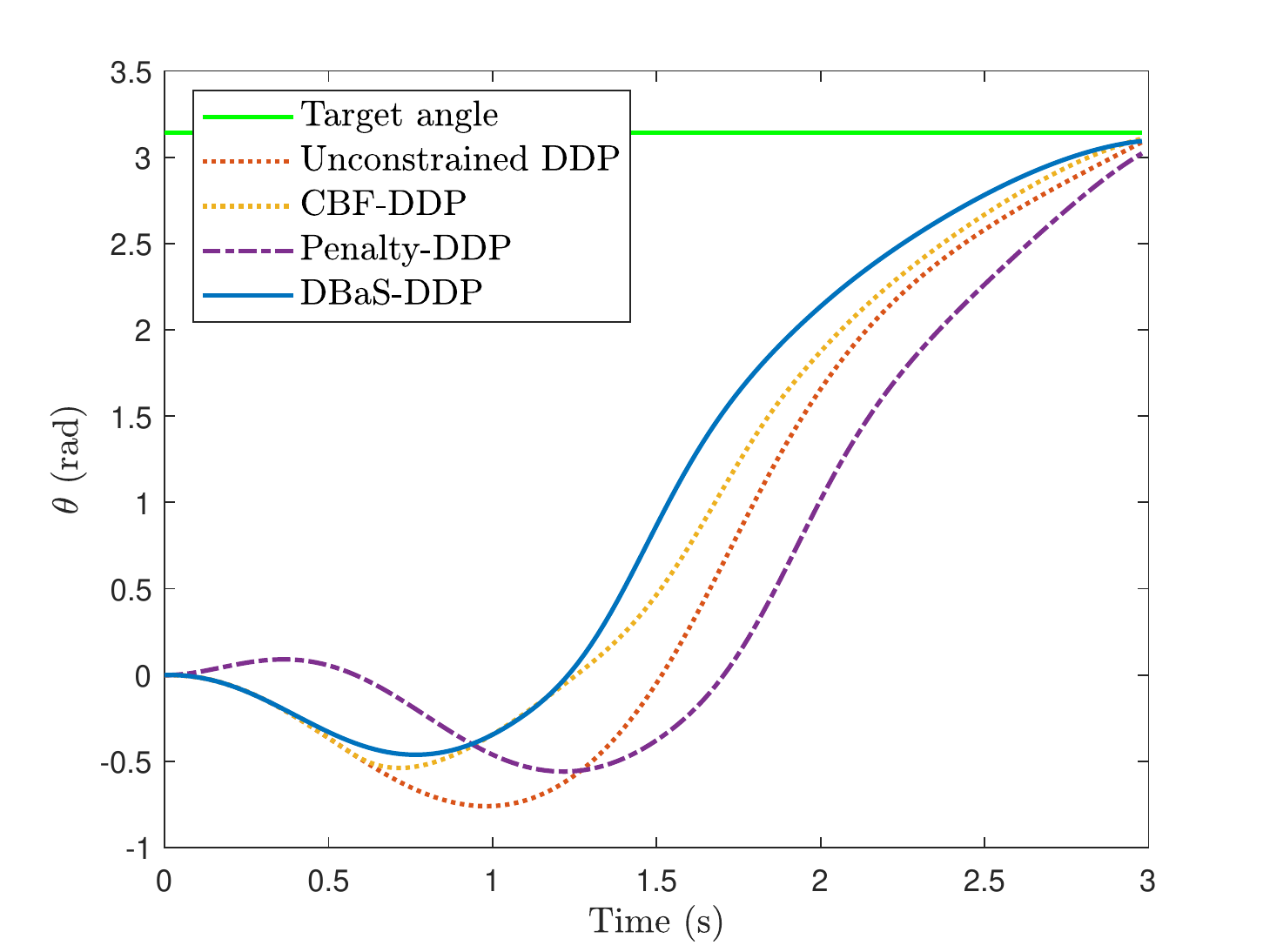}} \\
    \vspace{-4mm}
    \hspace*{-7mm}
    \subfloat{\includegraphics[trim=15 0 30 0, clip, width=0.5\linewidth]{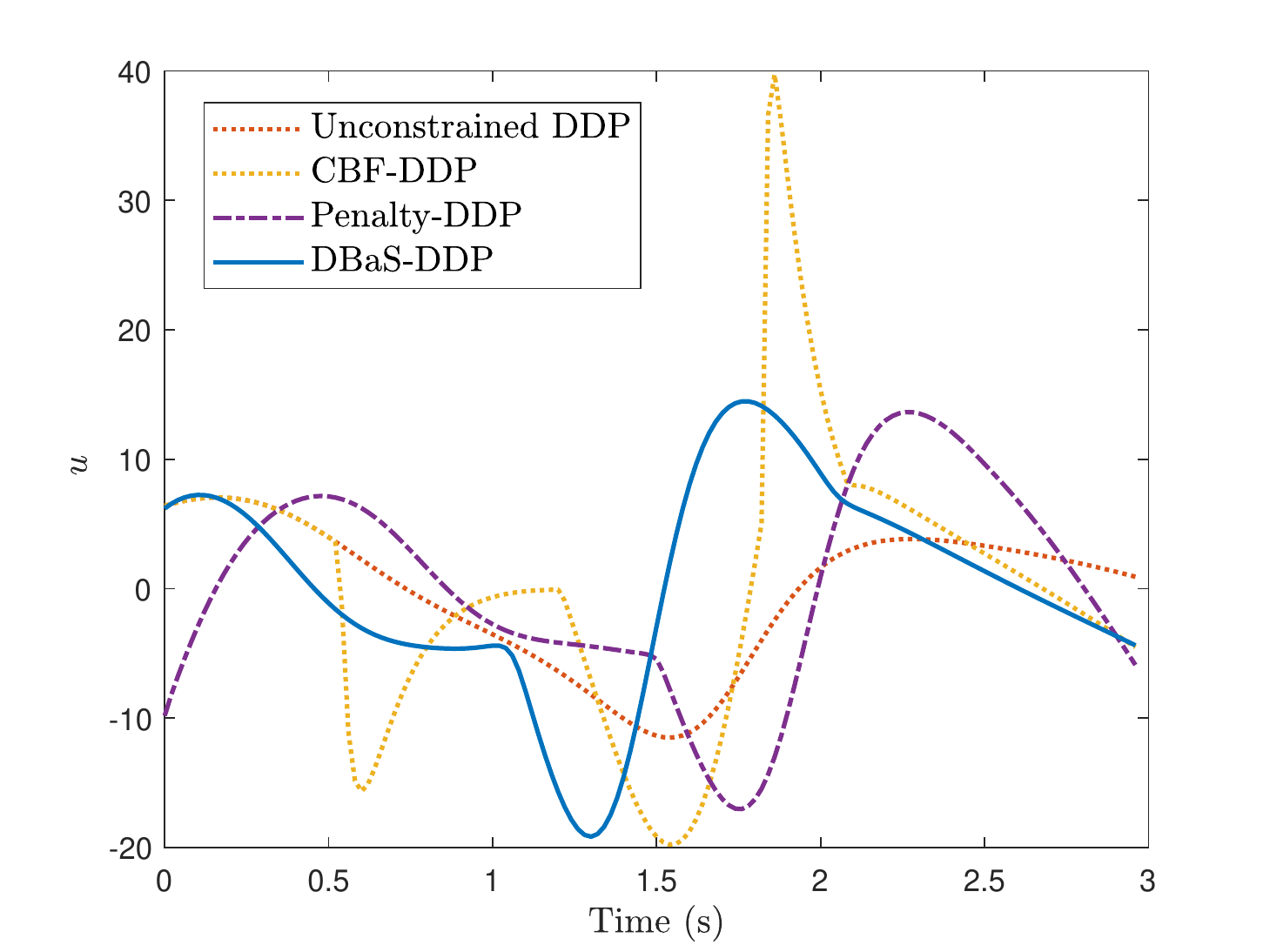}}
    \subfloat{\includegraphics[trim=15 0 30 0, clip, width=0.5\linewidth]{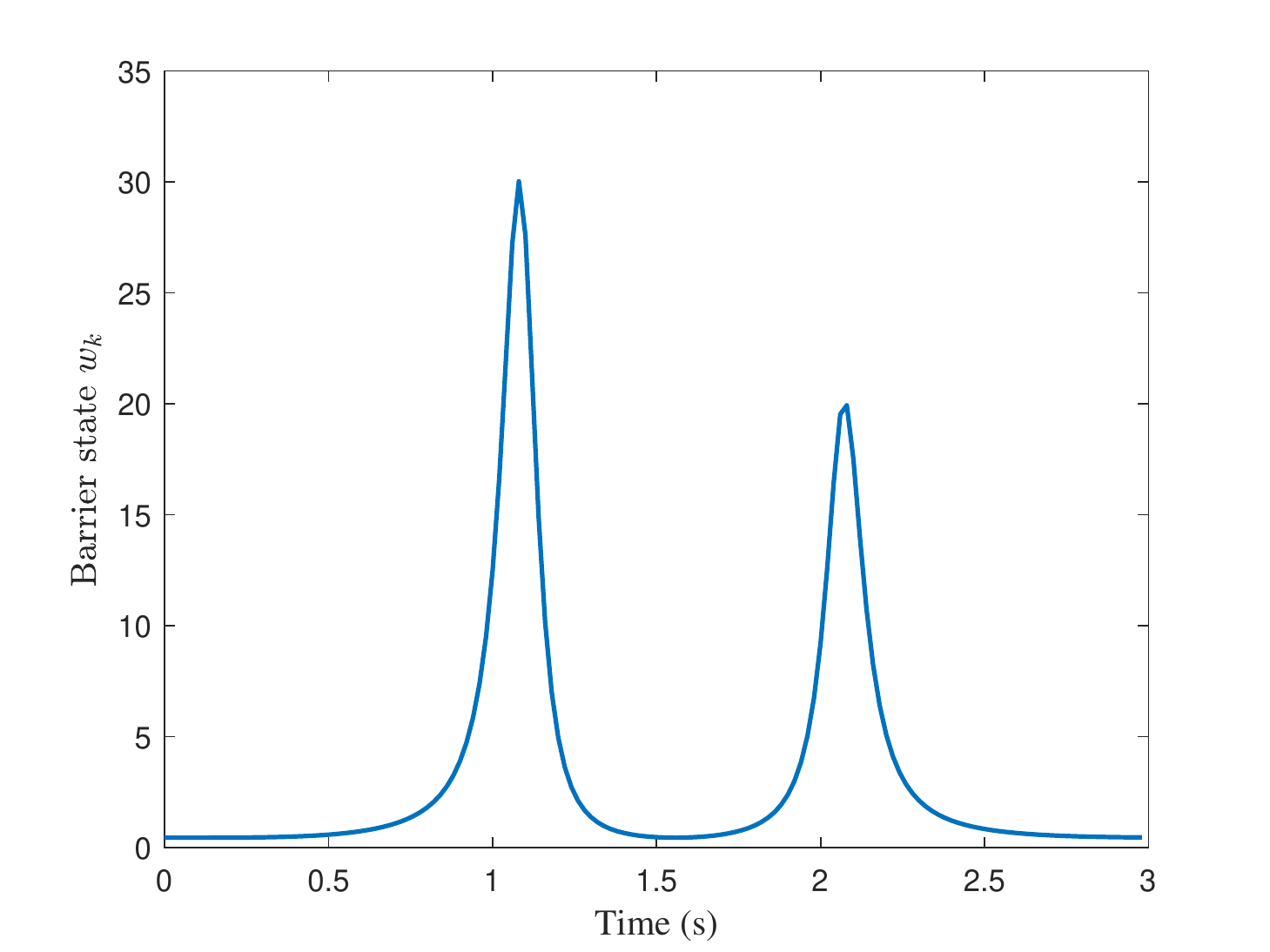}}
    \caption{Cart-pole swing up using unconstrained DDP (red), CBF-QP (yellow), penalty-DDP (purple) and DBaS-DDP (blue). DBaS-DDP respects the constraint on cart position (top left) while reaching the target angle (top right) with smooth control input (bottom left). The DBaS progression over time is shown in the bottom right sub-figure.}
    \label{fig:Cart pole}
    \vspace{-4mm}
\end{figure}

\subsection{Differential Wheeled Robot Safe Navigation}
The system dynamics are given by $\dot x = r\cos\theta(u_1 + u_2) / 2, \dot y = r\sin\theta(u_1 + u_2) / 2, \dot\theta = \frac{r}{2d}(u_1 - u_2)$,
where $x$ and $y$ are the robot's coordinates, $\theta$ is its heading, 
$r=0.2$ is the wheel radius, $d=0.2$ is the wheelbase width, and $u_1$ and $u_2$ are the speeds of the right and left wheels respectively. The robot is to safely navigate different courses including randomly generated obstacle courses, a tight course and a course with different geometric shapes. We used the inverse barrier function and augmented the system's dynamics with a single DBaS with $q_w=10^{-3}$ and $S = 100I_{3\times 3}$.  

\autoref{fig:diffwh_reaching} shows that DBaS-DDP can handle both simple and complex obstacles with barriers as defined in \autoref{shapes equations table}, while the naive penalty approach often fails. We also consider randomized courses. The start and target points are drawn from a uniform distribution such that the positions are within a $0.5$ unit square around $(3,0)$ and $(-3,0)$, and with angles up to $\pm 0.5$ rad. Anywhere from 1 to 10 obstacles are created with locations drawn from a normal distribution $\mathcal{N}(0,1)$ and radii drawn from $\mathbf{U} [0,1]$. Success is achieved if the generated trajectory reaches within $0.1$ units of the target. We performed 1000 trials for each number of obstacles, with results listed by obstacle count in \autoref{fig:DW_success_rate} and summarized in \autoref{tab:robustness} which show the DBaS approach consistently and significantly outperforming the penalty method. Note that because the relative degree is ill-conditioned for this system, we cannot compare results with standard CBF-based methods.
\begin{figure} [htb]
    \centering
     \hspace*{-0.5cm}\subfloat{\includegraphics[trim=0 5 20 20, clip, width=.5\linewidth]{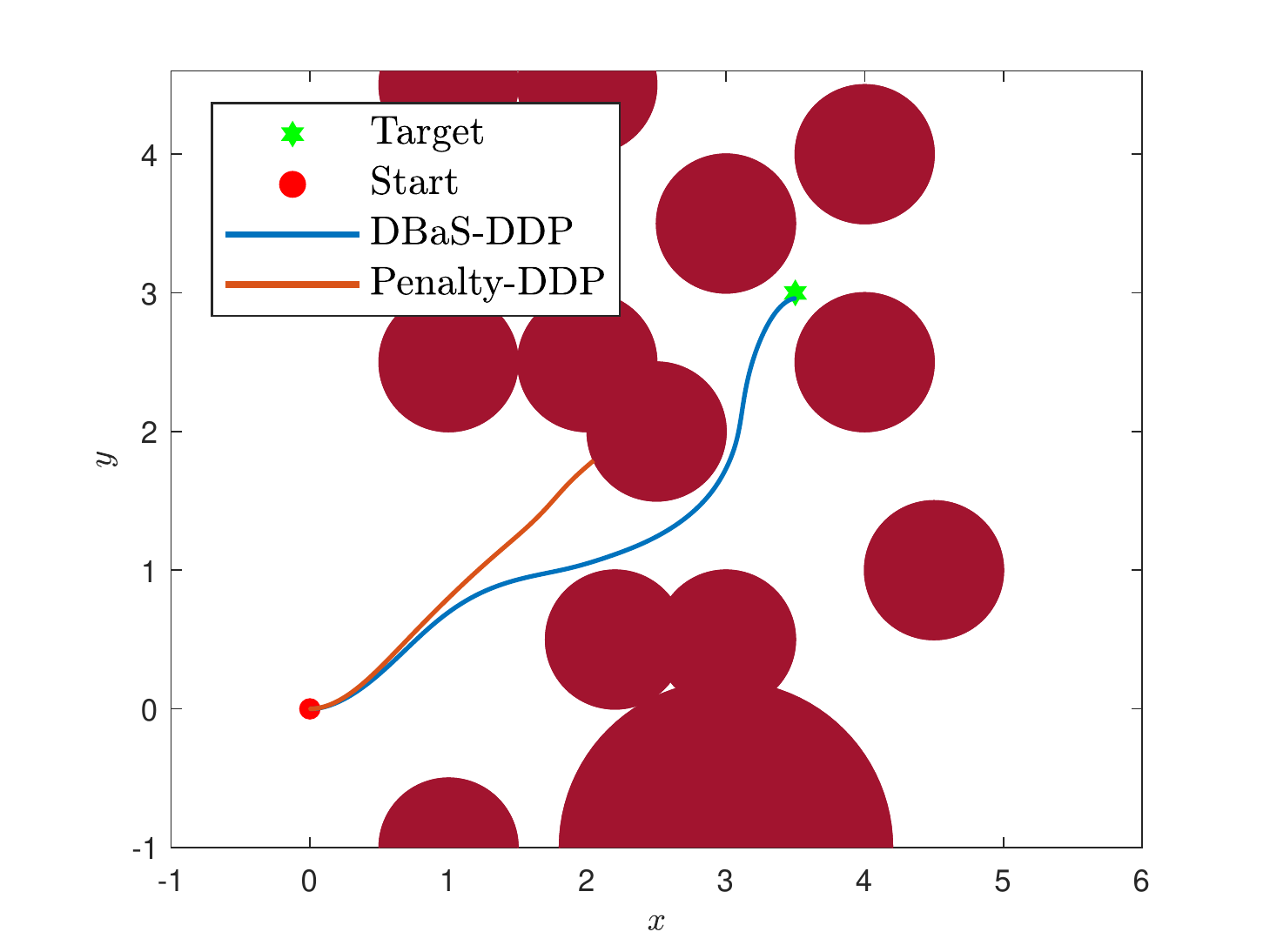}}
     \hspace*{-2mm}\subfloat{\includegraphics[trim=0 5 20 20, clip, width=.5\linewidth]{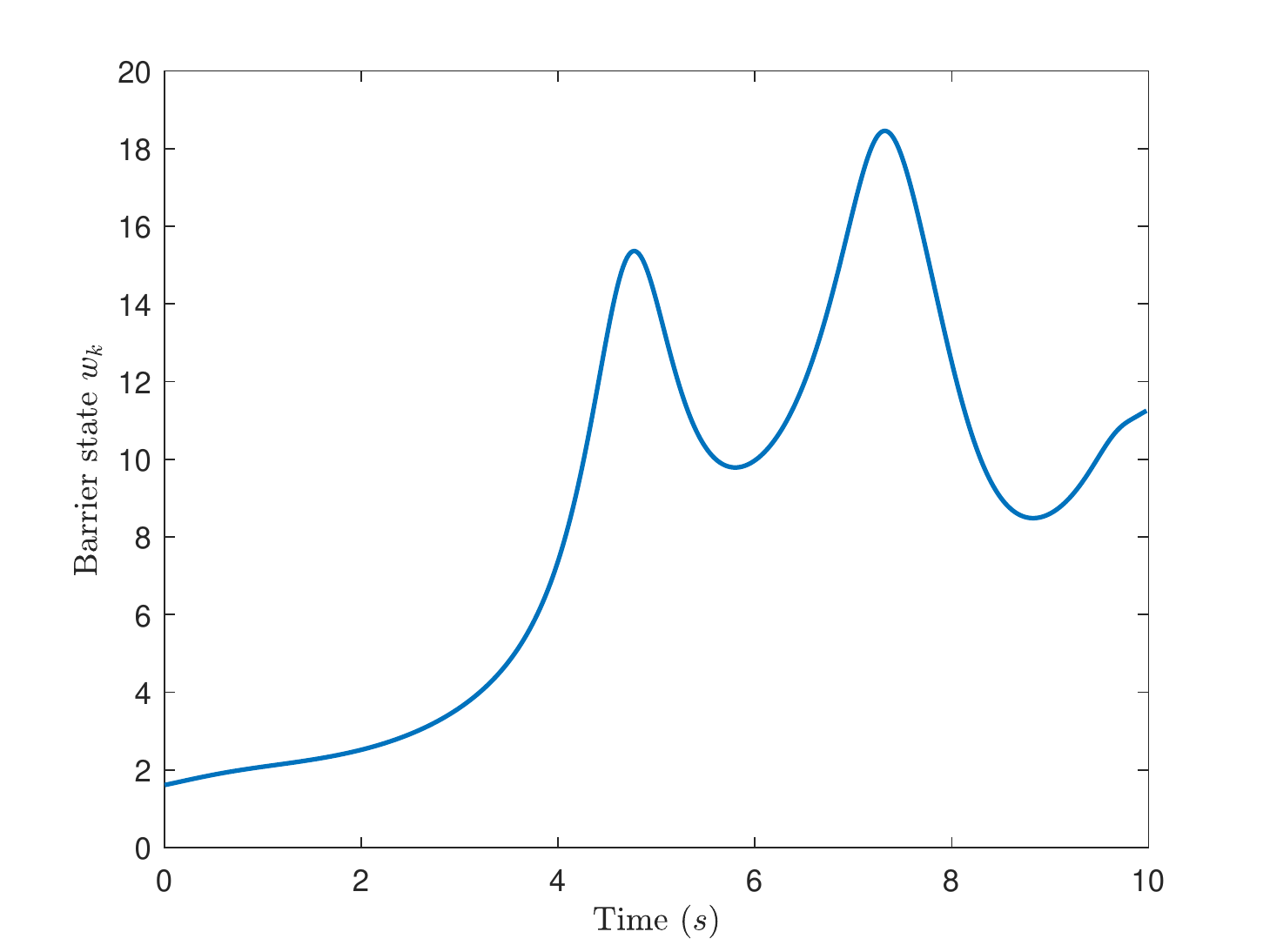}} 
     \\
     \hspace*{-0.5cm}\subfloat{\includegraphics[trim=0 5 20 20, clip, width=.5\linewidth]{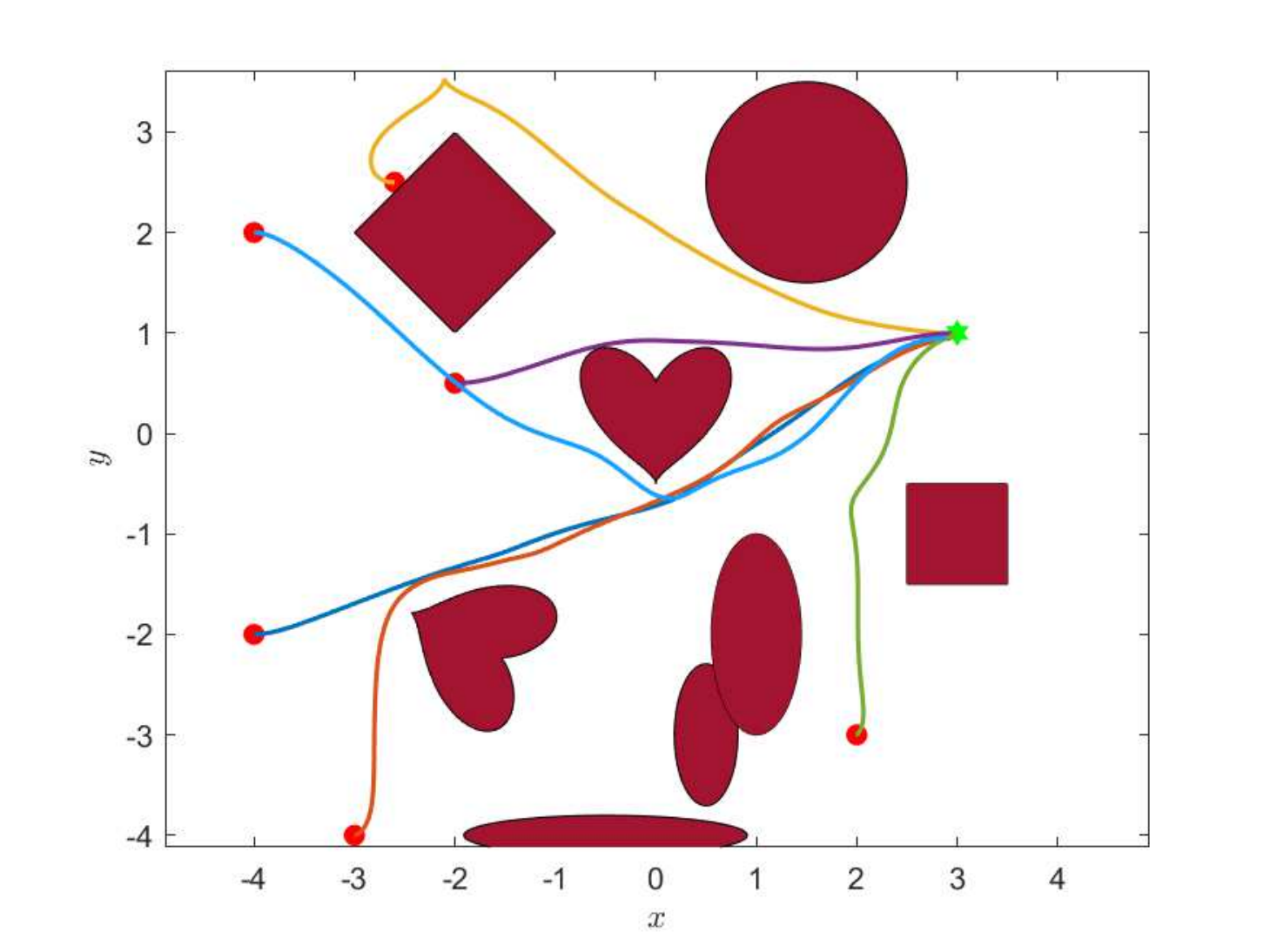}}
     \hspace*{-2mm}\subfloat{\includegraphics[trim=0 5 30 20, clip, width=.5\linewidth]{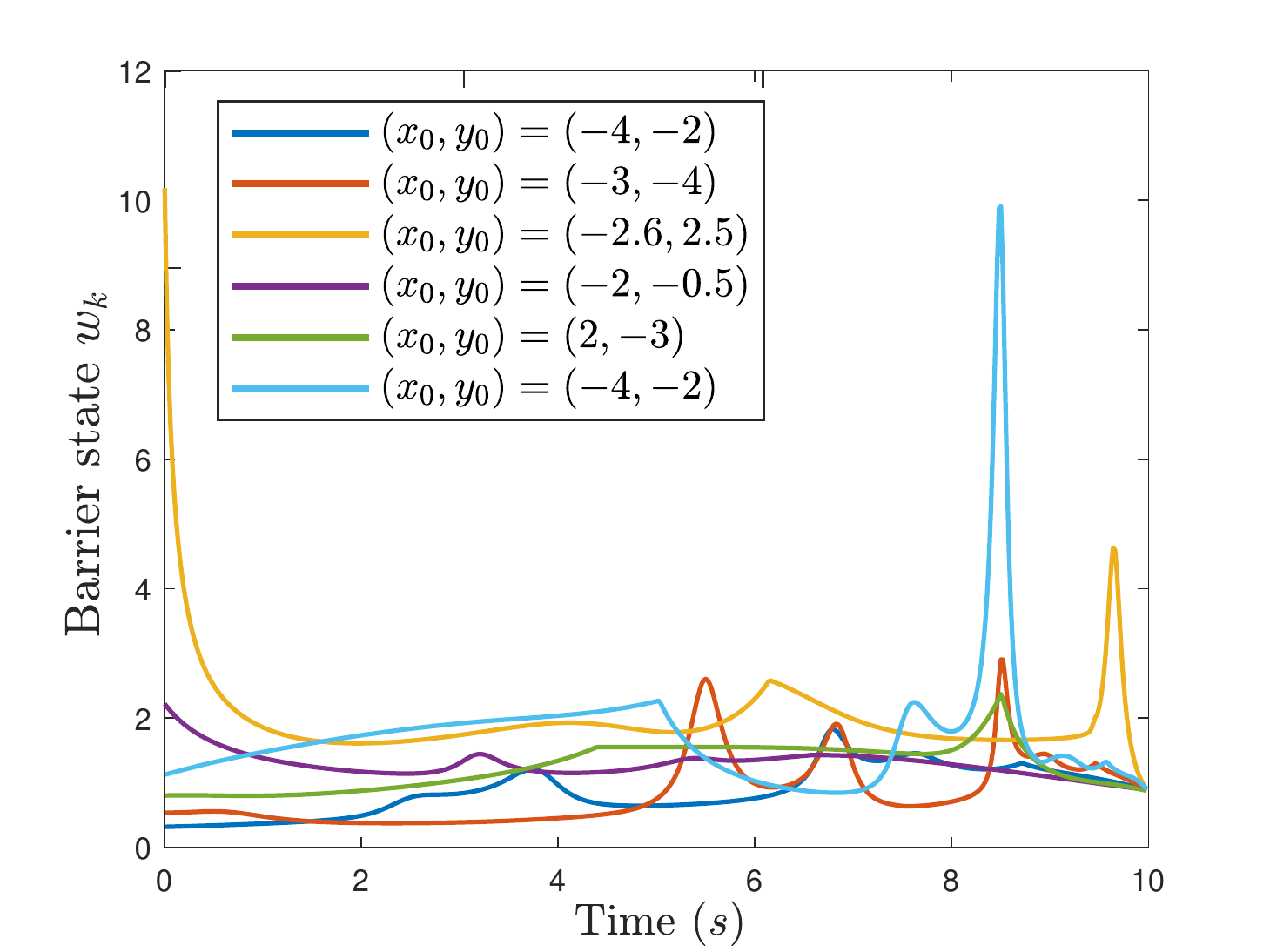}}
    \caption{Left: traces of DBaS-DDP and penalty-DDP (top) on differential wheeled robot, and several traces of DBaS-DDP with complex obstacles (bottom). Robots move from start (red) to goal (green). Right: progression of the associated barrier state over time in which larger values indicate that the robot is close to some obstacles.}
    \label{fig:diffwh_reaching}
     \vspace{-5mm}
\end{figure}

\begin{table}[htb]
    \caption{Equations used for complex obstacle shapes.}
    \centering
    \begin{tabular}{|c|c|}
    \hline
        Shape & Function \\
        \hline
        Ellipse & $a_x x^2 + a_y y^2 - r^2 $ \\
        Cardioid & $(a_x x^2+ a_y y^2-1)^3- a (a_x x)^2 ( a_y y)^3$ \\
        Diamond & $|x|+|y|$\\
        Square & $ |x+y|+|x-y| - 1$ \\
        \hline
    \end{tabular}
    \label{shapes equations table}
\end{table}

\begin{figure} [htb]
\vspace{-5mm}
    \centering
    \includegraphics[trim=0 5 40 0, clip, width=0.7\linewidth]{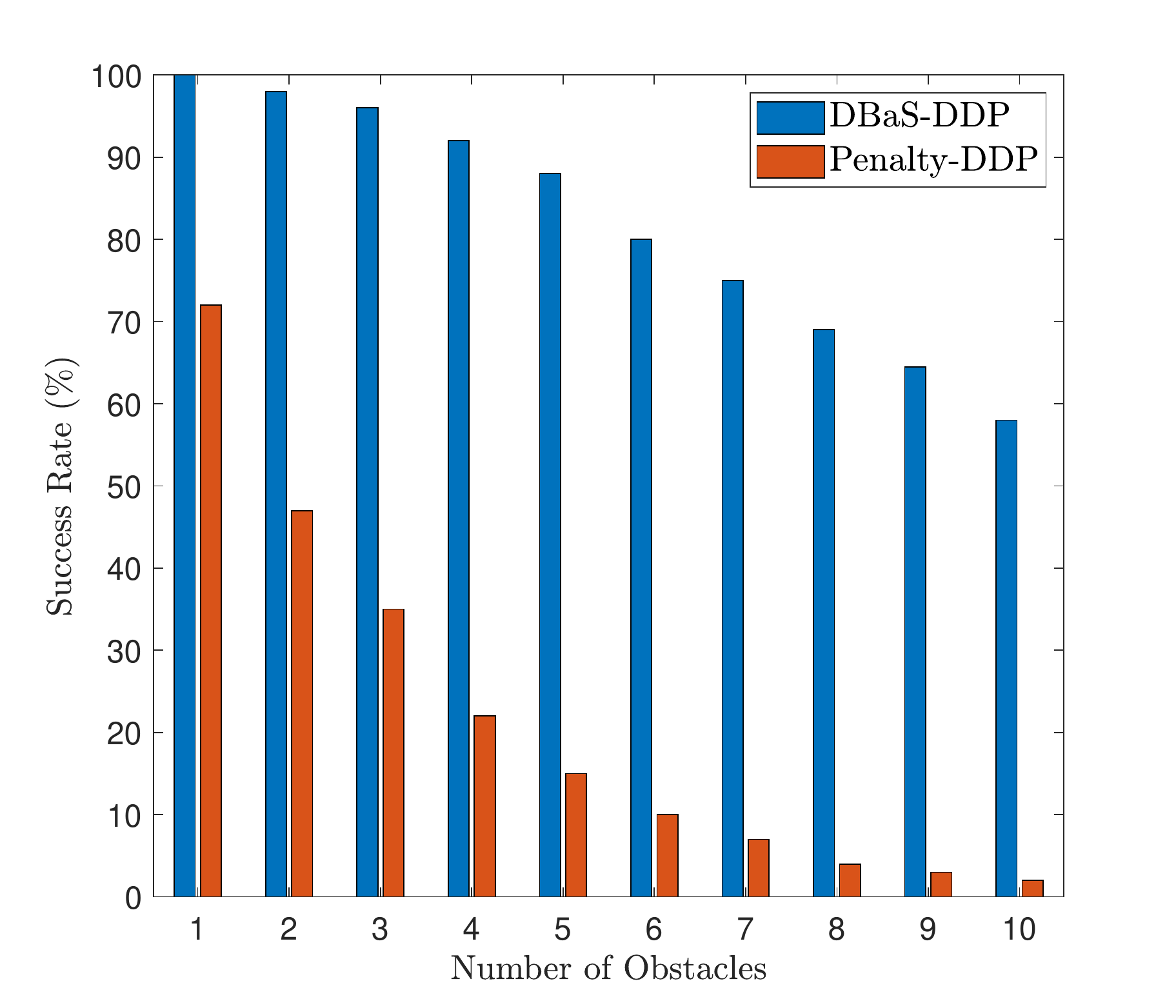}
    \caption{Success rates of DBaS-DDP (blue) and the penalty method (red) on the differential wheeled robot example in randomized obstacle courses with varying obstacle counts.}
    \label{fig:DW_success_rate}
    \vspace{-5mm}
\end{figure}

\subsection{Quadrotor Safe Reaching and Tracking}
\subsubsection{Reaching Task}
We applied the discrete barrier state based DDP (DBaS-DDP) to a quadrotor model as described in \citet{Sabatino2015QuadrotorCM} with unity parameters ($1$kg, $1\text{kg\,m}^2$, etc.). The quadrotor was to perform a \textit{reaching} problem safely, i.e. to fly from some initial state to some arbitrary final state in the presence of some obstacles without collision. The safe set is again defined as the complement of a set of spherical obstacles. A solution to the quadrotor reaching problem found by DBaS-DDP with randomly-generated obstacles is shown in \autoref{fig:quad_reaching}.

    \begin{figure}[h]
         \vspace{-5mm}
   \centering
     \subfloat{\includegraphics[trim=0 0 0 0, clip, width=0.53\linewidth]{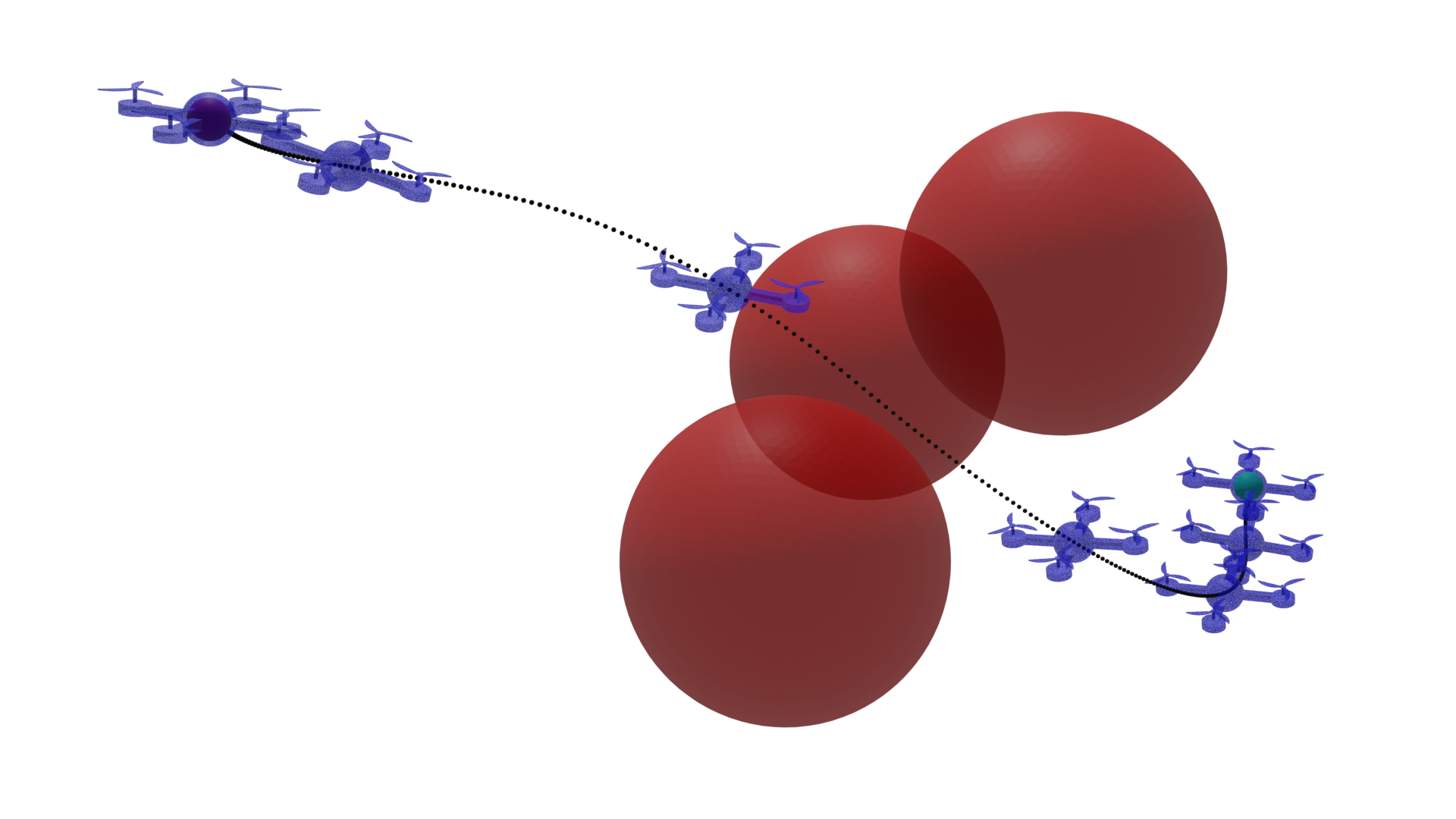}}
    \subfloat{\includegraphics[trim=0 0 0 0, clip, width=0.6\linewidth]{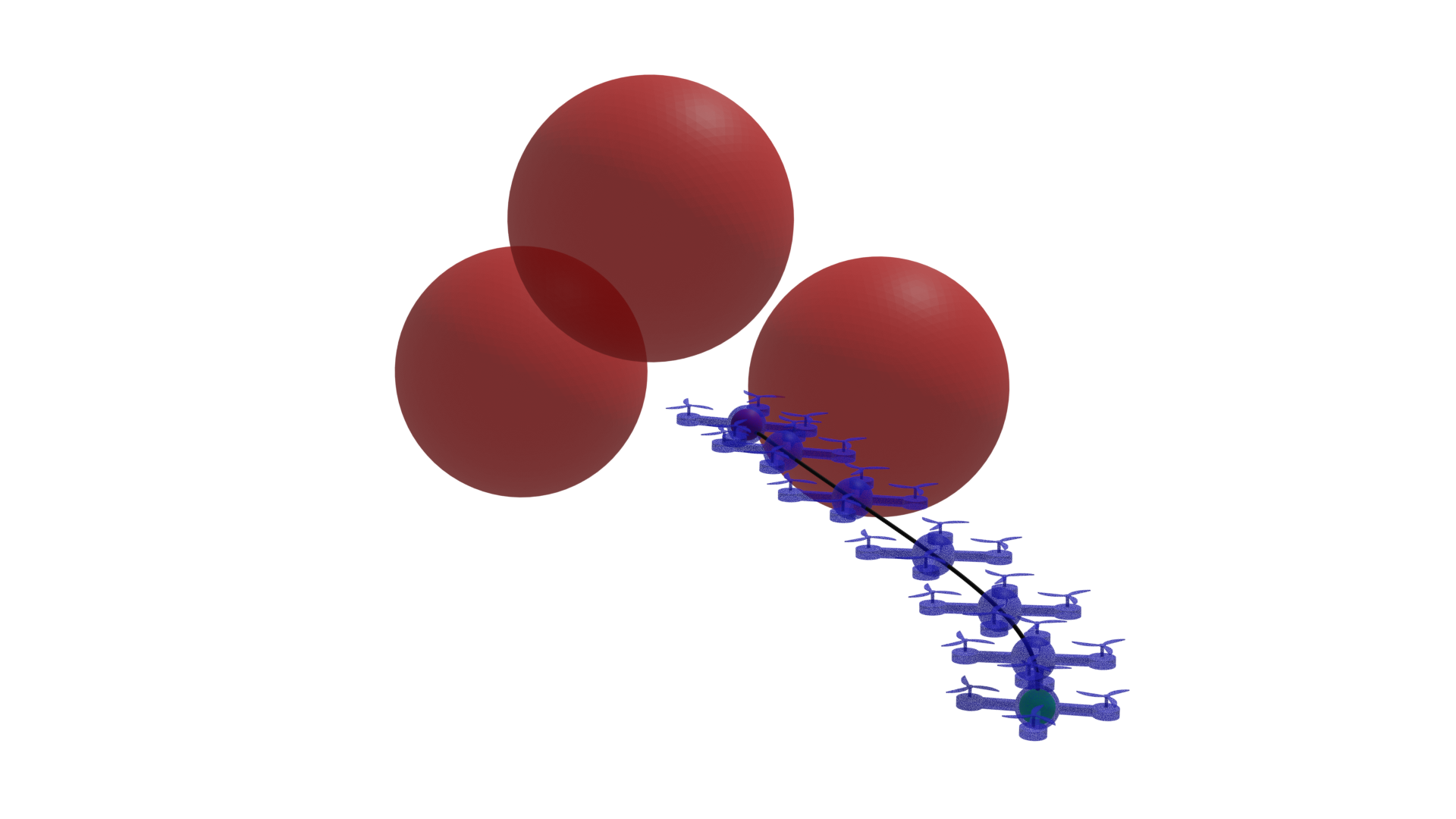}}
    \caption{Quadrotor reaching task with tight squeeze using DBaS-DDP (left) compared to the penalty method (right). Using the DBaS-DDP solver, the quadrotor was able to reach the goal (dark ball) safely while the penalty based DDP solver failed to navigate through the obstacles.}
      \label{reaching_tight}
   \end{figure}   

DBaS-DDP was compared against penalty-DDP in two fixed environments: a single-obstacle case where the quadrotor must navigate around a single large spherical obstacle and a three-obstacle case (shown in \autoref{reaching_tight}) designed to add a local minimum and a narrow passage between the obstacles. When testing in these fixed environments with random initial conditions, we see that DBaS-DDP successfully finds a path to the goal much more frequently than penalty-DDP as shown in \autoref{tab:robustness}. Note that \textit{failure} in this case indicates failure to reach the goal, rather than a failure to maintain safety: as described in \autoref{prop:safety}, the trajectory and controller found by DBaS-DDP is guaranteed to be safe in all cases.

DBaS-DDP was also tested in the presence of 40 randomized obstacles in an environment similar to \autoref{fig:quad_reaching}. In this case, DBaS-DDP reached the goal 96\% of the time while the penalty-based method only succeeded 59\% of the time and found substantially higher-cost trajectories.

While there exist CBF-based obstacle avoidance methods for quadrotors, they are tedious to construct for the full 3D model and comparison is out of the scope of this analysis.

\subsubsection{Tracking Task}
The technique of barrier states can also generalize to the \textit{tracking} problem, in which we want to safely track some (possibly unsafe) reference trajectory. To put the DBaS-DDP to the test, we attempted to track the trajectory defined by the parametric equations for a figure eight:
$$
\rm{x}(s) = \sin(2s), \rm{y}(s) = \cos(s), \rm{z}(s) = 0, s(t) = \frac{(\pi t / 25)^2}{\pi t / 25 + 1}
$$
Then, the squared deviation of the quadrotor's trajectory from this path was penalized in the cost function. In addition, we placed an obstacle at the origin forcing the quadrotor to navigate around the obstacle to remain safe. \autoref{fig:quad_tracking} shows an execution trace from this experiment. The quadrotor was able to successfully track the trajectory in a very aggressive maneuver without losing safety.

\begin{figure} [h]
    \centering
    \includegraphics[trim=200 100 200 100, clip, width=.66\linewidth]{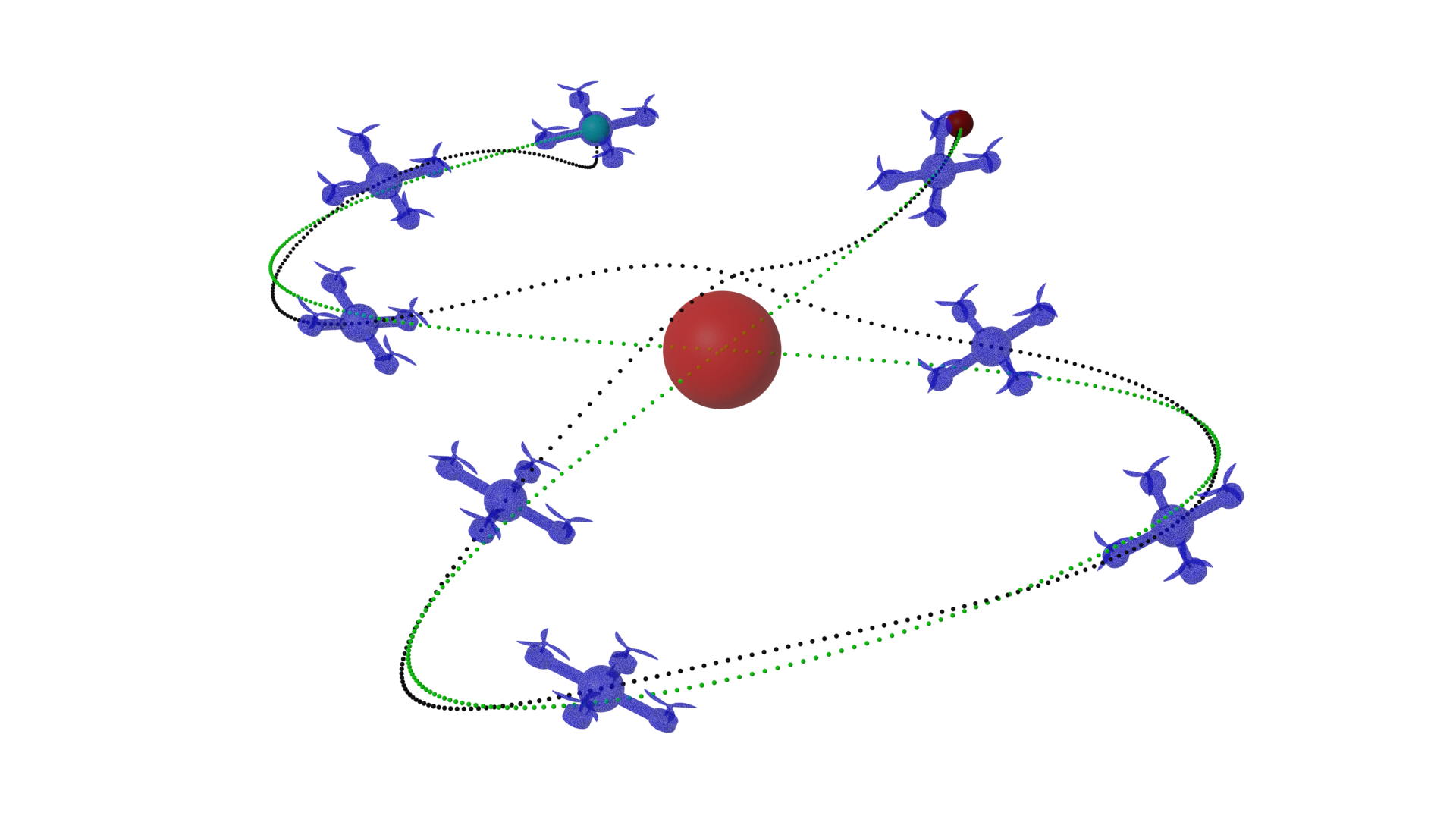}
    \caption{Quadrotor tracking a predetermined trajectory (green) while avoiding a spherical obstacle. Starting from the green ball, the safe trajectory from DBaS-DDP (black) successfully avoids the obstacle to reach the red ball.}
    \label{fig:quad_tracking}
    \vspace{-3mm}
\end{figure}

\section{Conclusion} \label{Section: Conclusion}
In this work, the newly proposed barrier state method for stabilization of continuous time systems was extended to trajectory optimization of discrete time systems. This extension, named the discrete barrier state (DBaS) method, provides provable safety guarantees when combined with DDP for safe trajectory optimization problems. 
To show the efficacy of the proposed safety embedded DDP, we presented several comparisons with other commonly used methods and successful simulation examples for a constrained cart-pole swing up, safe holonomic and non-holonomic robot navigations, and a quadrotor performing safety-critical planning and execution, demonstrating improvements in comparison to the other methods on each problem.

Our work requires perfect knowledge of the system's dynamics and, similarly, assumes full knowledge of state and safety constraints, which may not hold true in real-world applications where the former may require model identification or leaning and the latter are recovered from sensor measurements. Incorporating dynamics, state, and safety constraint uncertainty, for example using Gaussian Process regression, into the DBaS framework represents a promising direction for future research. Furthermore, future work will include improving robustness by extending the DBaS-DDP to min-max and risk-sensitive optimal control problems. Additionally, we are currently developing real-time implementations of DBaS-DDP in a lower-level language and plan to conduct physical experiments using a receding-horizon formulation.
    \vspace{-1mm}

\bibliographystyle{IEEEtranN}
\bibliography{IEEEabrv,references}

\begin{thebibliography}{32}
\providecommand{\natexlab}[1]{#1}
\providecommand{\url}[1]{#1}
\csname url@samestyle\endcsname
\providecommand{\newblock}{\relax}
\providecommand{\bibinfo}[2]{#2}
\providecommand{\BIBentrySTDinterwordspacing}{\spaceskip=0pt\relax}
\providecommand{\BIBentryALTinterwordstretchfactor}{4}
\providecommand{\BIBentryALTinterwordspacing}{\spaceskip=\fontdimen2\font plus
\BIBentryALTinterwordstretchfactor\fontdimen3\font minus
  \fontdimen4\font\relax}
\providecommand{\BIBforeignlanguage}[2]{{%
\expandafter\ifx\csname l@#1\endcsname\relax
\typeout{** WARNING: IEEEtranN.bst: No hyphenation pattern has been}%
\typeout{** loaded for the language `#1'. Using the pattern for}%
\typeout{** the default language instead.}%
\else
\language=\csname l@#1\endcsname
\fi
#2}}
\providecommand{\BIBdecl}{\relax}
\BIBdecl

\bibitem[Almubarak et~al.(2022)Almubarak, Sadegh, and
  Theodorou]{Almubarak2021SafetyEC}
H.~Almubarak, N.~Sadegh, and E.~A. Theodorou,
  ``\href{https://ieeexplore.ieee.org/document/9467052}{Safety Embedded Control
  of Nonlinear Systems via Barrier States},'' \emph{IEEE Control Systems
  Letters}, vol.~6, pp. 1328--1333, 2022.

\bibitem[Blanchini(1999)]{blanchini1999set}
F.~Blanchini,
  ``\href{https://www.sciencedirect.com/science/article/pii/S0005109899001132}{Set
  invariance in control},'' \emph{Automatica}, vol.~35, no.~11, pp. 1747--1767,
  1999.

\bibitem[Prajna(2003)]{prajna2003barrier}
S.~Prajna, ``\href{https://ieeexplore.ieee.org/document/1273063}{Barrier
  certificates for nonlinear model validation},'' in \emph{42nd IEEE
  International Conference on Decision and Control (IEEE Cat. No. 03CH37475)},
  vol.~3.\hskip 1em plus 0.5em minus 0.4em\relax IEEE, 2003, pp. 2884--2889.

\bibitem[Prajna and Jadbabaie(2004)]{prajna2004safety}
S.~Prajna and A.~Jadbabaie,
  ``\href{https://link.springer.com/chapter/10.1007/978-3-540-24743-2_32}{Safety
  verification of hybrid systems using barrier certificates},'' in
  \emph{International Workshop on Hybrid Systems: Computation and
  Control}.\hskip 1em plus 0.5em minus 0.4em\relax Springer, 2004, pp.
  477--492.

\bibitem[Wieland and Allg{\"o}wer(2007)]{wieland2007constructive}
P.~Wieland and F.~Allg{\"o}wer,
  ``\href{https://www.sciencedirect.com/science/article/pii/S1474667016355690}{Constructive
  safety using control barrier functions},'' \emph{IFAC Proceedings Volumes},
  vol.~40, no.~12, pp. 462--467, 2007.

\bibitem[Ames et~al.(2014)Ames, Grizzle, and Tabuada]{ames2014control}
A.~D. Ames, J.~W. Grizzle, and P.~Tabuada,
  ``\href{https://ieeexplore.ieee.org/document/7040372}{Control barrier
  function based quadratic programs with application to adaptive cruise
  control},'' in \emph{53rd IEEE Conference on Decision and Control}.\hskip 1em
  plus 0.5em minus 0.4em\relax IEEE, 2014, pp. 6271--6278.

\bibitem[Romdlony and Jayawardhana(2014)]{romdlony2014uniting}
M.~Z. Romdlony and B.~Jayawardhana,
  ``\href{https://ieeexplore.ieee.org/document/7039737}{Uniting control
  Lyapunov and control barrier functions},'' in \emph{53rd IEEE Conference on
  Decision and Control}.\hskip 1em plus 0.5em minus 0.4em\relax IEEE, 2014, pp.
  2293--2298.

\bibitem[Ames et~al.(2016)Ames, Xu, Grizzle, and
  Tabuada]{ames2016CBF-forSaferyCritControl}
A.~D. Ames, X.~Xu, J.~W. Grizzle, and P.~Tabuada,
  ``\href{https://ieeexplore.ieee.org/document/7782377}{Control barrier
  function based quadratic programs for safety critical systems},'' \emph{IEEE
  Transactions on Automatic Control}, vol.~62, no.~8, pp. 3861--3876, 2016.

\bibitem[Agrawal and Sreenath(2017)]{agrawal2017discrete}
A.~Agrawal and K.~Sreenath,
  ``\href{https://rss2017.lids.mit.edu/program/papers/22/}{Discrete Control
  Barrier Functions for Safety-Critical Control of Discrete Systems with
  Application to Bipedal Robot Navigation.}'' in \emph{Robotics: Science and
  Systems}, 2017.

\bibitem[Choi et~al.(2020)Choi, Castaneda, Tomlin, and
  Sreenath]{choi2020reinforcement}
J.~Choi, F.~Castaneda, C.~J. Tomlin, and K.~Sreenath,
  ``\href{https://roboticsconference.org/2020/program/papers/88.html}{Reinforcement
  learning for safety-critical control under model uncertainty, using control
  lyapunov functions and control barrier functions},'' in \emph{Robotics:
  Science and Systems}, 2020.

\bibitem[Taylor and Ames(2020)]{taylor2020adaptive}
A.~J. Taylor and A.~D. Ames,
  ``\href{https://ieeexplore.ieee.org/abstract/document/9147463}{Adaptive
  safety with control barrier functions},'' in \emph{2020 American Control
  Conference (ACC)}.\hskip 1em plus 0.5em minus 0.4em\relax IEEE, 2020, pp.
  1399--1405.

\bibitem[Wang et~al.(2018)Wang, Theodorou, and Egerstedt]{wang2018safe}
L.~Wang, E.~A. Theodorou, and M.~Egerstedt,
  ``\href{https://ieeexplore.ieee.org/document/8460471}{Safe learning of
  quadrotor dynamics using barrier certificates},'' in \emph{2018 IEEE
  International Conference on Robotics and Automation (ICRA)}.\hskip 1em plus
  0.5em minus 0.4em\relax IEEE, 2018, pp. 2460--2465.

\bibitem[Ahmadi et~al.(2019)Ahmadi, Singletary, Burdick, and
  Ames]{ahmadi2019safe}
M.~Ahmadi, A.~Singletary, J.~W. Burdick, and A.~D. Ames,
  ``\href{https://ieeexplore.ieee.org/document/9030241}{Safe policy synthesis
  in multi-agent pomdps via discrete-time barrier functions},'' in \emph{2019
  IEEE 58th Conference on Decision and Control (CDC)}.\hskip 1em plus 0.5em
  minus 0.4em\relax IEEE, 2019, pp. 4797--4803.

\bibitem[Nguyen and Sreenath(2016)]{nguyen2016exponential}
Q.~Nguyen and K.~Sreenath,
  ``\href{https://ieeexplore.ieee.org/document/7524935}{Exponential Control
  Barrier Functions for enforcing high relative-degree safety-critical
  constraints},'' in \emph{2016 American Control Conference (ACC)}, 2016, pp.
  322--328.

\bibitem[Xiao and Belta(2019)]{xiao2019control}
W.~Xiao and C.~Belta,
  ``\href{https://ieeexplore.ieee.org/document/9029455}{Control Barrier
  Functions for Systems with High Relative Degree},'' \emph{2019 IEEE 58th
  Conference on Decision and Control (CDC)}, pp. 474--479, 2019.

\bibitem[Pereira et~al.(2020)Pereira, Wang, Exarchos, and
  Theodorou]{marcus2020safe}
M.~A. Pereira, Z.~Wang, I.~Exarchos, and E.~A. Theodorou,
  ``\href{https://arxiv.org/abs/2009.01196}{Safe optimal control using
  stochastic barrier functions and deep forward-backward sdes},'' \emph{arXiv
  preprint arXiv:2009.01196}, 2020.

\bibitem[Long et~al.(2021)Long, Qian, Cort{\'e}s, and
  Atanasov]{long2021learning}
K.~Long, C.~Qian, J.~Cort{\'e}s, and N.~Atanasov,
  ``\href{https://ieeexplore.ieee.org/document/9392327}{Learning Barrier
  Functions with Memory for Robust Safe Navigation},'' \emph{IEEE Robotics and
  Automation Letters}, vol.~6, no.~3, pp. 4931--4938, 2021.

\bibitem[Xiao et~al.(2020)Xiao, Belta, and Cassandras]{xiao2020feasibility}
W.~Xiao, C.~A. Belta, and C.~G. Cassandras,
  ``\href{https://ieeexplore.ieee.org/document/9303857}{Feasibility-guided
  learning for constrained optimal control problems},'' in \emph{2020 59th IEEE
  Conference on Decision and Control (CDC)}.\hskip 1em plus 0.5em minus
  0.4em\relax IEEE, 2020, pp. 1896--1901.

\bibitem[Wang et~al.(2016)Wang, Ames, and Egerstedt]{wang2016multi}
L.~Wang, A.~D. Ames, and M.~Egerstedt,
  ``\href{https://ieeexplore.ieee.org/abstract/document/7798663?casa_token=QVaIW3ifv_AAAAAA:RIBdtKiBgkOi9mBjYf8b6hcLki6ecasSVfuxT9WPThsxQpvwDhEZKw5hGYY3wNwHDo_HKwgYww}{Multi-objective
  compositions for collision-free connectivity maintenance in teams of mobile
  robots},'' in \emph{2016 IEEE 55th Conference on Decision and Control
  (CDC)}.\hskip 1em plus 0.5em minus 0.4em\relax IEEE, 2016, pp. 2659--2664.

\bibitem[Murray and Yakowitz(1979)]{murray1979constrainedDDP}
D.~M. Murray and S.~J. Yakowitz,
  ``\href{https://agupubs.onlinelibrary.wiley.com/doi/abs/10.1029/WR015i005p01017}{Constrained
  differential dynamic programming and its application to multireservoir
  control},'' \emph{Water Resources Research}, vol.~15, no.~5, pp. 1017--1027,
  1979.

\bibitem[{Tassa} et~al.(2014){Tassa}, {Mansard}, and
  {Todorov}]{tassa2014control}
Y.~{Tassa}, N.~{Mansard}, and E.~{Todorov}, ``Control-limited differential
  dynamic programming,'' in
  \emph{\href{https://ieeexplore.ieee.org/document/6907001}{2014 IEEE
  International Conference on Robotics and Automation (ICRA)}}, 2014, pp.
  1168--1175.

\bibitem[{Plancher} et~al.(2017){Plancher}, {Manchester}, and
  {Kuindersma}]{plancher2017constrained}
B.~{Plancher}, Z.~{Manchester}, and S.~{Kuindersma},
  ``\href{https://ieeexplore.ieee.org/document/8206457}{Constrained unscented
  dynamic programming},'' in \emph{2017 IEEE/RSJ International Conference on
  Intelligent Robots and Systems (IROS)}, 2017, pp. 5674--5680.

\bibitem[{Howell} et~al.(2019){Howell}, {Jackson}, and
  {Manchester}]{howell2019altro}
T.~A. {Howell}, B.~E. {Jackson}, and Z.~{Manchester},
  ``\href{https://ieeexplore.ieee.org/document/8967788}{ALTRO: A Fast Solver
  for Constrained Trajectory Optimization},'' in \emph{2019 IEEE/RSJ
  International Conference on Intelligent Robots and Systems (IROS)}, 2019, pp.
  7674--7679.

\bibitem[Xie et~al.(2017)Xie, Liu, and Hauser]{xie2017ddpNonlinearconstraints}
Z.~Xie, C.~K. Liu, and K.~Hauser,
  ``\href{https://ieeexplore.ieee.org/document/7989086}{Differential dynamic
  programming with nonlinear constraints},'' in \emph{2017 IEEE International
  Conference on Robotics and Automation (ICRA)}.\hskip 1em plus 0.5em minus
  0.4em\relax IEEE, 2017, pp. 695--702.

\bibitem[Aoyama et~al.(2020)Aoyama, Boutselis, Patel, and
  Theodorou]{aoyama2020constrainedDDP}
Y.~Aoyama, G.~Boutselis, A.~Patel, and E.~A. Theodorou,
  ``\href{https://arxiv.org/abs/2005.00985}{Constrained Differential Dynamic
  Programming Revisited},'' \emph{arXiv preprint arXiv:2005.00985}, 2020.

\bibitem[Pavlov et~al.(2021)Pavlov, Shames, and
  Manzie]{pavlov2021interiorConstrainedDDP}
A.~Pavlov, I.~Shames, and C.~Manzie,
  ``\href{https://ieeexplore.ieee.org/document/9332234}{Interior point
  differential dynamic programming},'' \emph{IEEE Transactions on Control
  Systems Technology}, 2021.

\bibitem[Ben-Tal and Nemirovski(2020)]{ben2020lecturesoptimization}
A.~Ben-Tal and A.~Nemirovski,
  \emph{\href{https://www2.isye.gatech.edu/~nemirovs/LMCOLN2020.pdf}{Lectures
  on modern convex optimization}}, 2020.

\bibitem[Wills and Heath(2004)]{wills2004barrier}
A.~G. Wills and W.~P. Heath,
  ``\href{https://www.sciencedirect.com/science/article/pii/S0005109804000809}{Barrier
  function based model predictive control},'' \emph{Automatica}, vol.~40,
  no.~8, pp. 1415--1422, 2004.

\bibitem[Mayne(1966)]{mayne1966second}
D.~Mayne,
  ``\href{https://www.tandfonline.com/doi/abs/10.1080/00207176608921369}{A
  second-order gradient method for determining optimal trajectories of
  non-linear discrete-time systems},'' \emph{International Journal of Control},
  vol.~3, no.~1, pp. 85--95, 1966.

\bibitem[Jacobson(1967)]{jacobson1967differential-PhDthesis}
D.~H. Jacobson,
  ``\href{https://spiral.imperial.ac.uk/bitstream/10044/1/17351/2/Jacobson-DH-1967-PhD-Thesis.pdf}{Differential
  dynamic programming methods for determining optimal control of non-linear
  systems},'' Ph.D. dissertation, Electrical Engineering, Centre for Computing
  and Automation, Imperial College of Science and Technology, University of
  London., 1967.

\bibitem[Jacobson and Mayne(1970)]{jacobson1970differential}
D.~H. Jacobson and D.~Q. Mayne, \emph{Differential dynamic programming}.\hskip
  1em plus 0.5em minus 0.4em\relax North-Holland, 1970.

\bibitem[Sabatino(2015)]{Sabatino2015QuadrotorCM}
F.~Sabatino,
  ``\href{https://www.diva-portal.org/smash/record.jsf?pid=diva2\%3A860649&dswid=1928}{Quadrotor
  control: modeling, nonlinearcontrol design, and simulation},'' 2015.

\end{thebibliography}

\vfill

\end{document}